%% file: median_main.tex
\newif\iffull
\fulltrue
\iffull
\documentclass[11pt,twoside]{article}
\usepackage{amsmath,amsthm,amssymb,fullpage,comment,bbm,algorithm,algorithmic,framed,hyperref}
\usepackage{amsmath,color}
\else
\documentclass{colt2017}
\fi
\usepackage[utf8]{inputenc}
\usepackage{vfmacros-mod,framed,algorithm,algorithmic,bbm}

\iffull
\usepackage[style=alphabetic,backend=bibtex,maxbibnames=20,maxcitenames=6,firstinits=true,doi=false,url=false]{biblatex}
\newcommand*{\citet}[1]{\AtNextCite{\AtEachCitekey{\defcounter{maxnames}{2}}} \textcite{#1}}

\newcommand*{\citep}[1]{\cite{#1}}
\newcommand{\citeyearpar}[1]{\cite{#1}}
\bibliography{refs,vf-allrefs}
\fi

\newcommand{\INDSTATE}[1][1]{\STATE\hspace{#1\algorithmicindent}}

\newcommand{\ex}[2]{{\ifx&#1& \E \else \E_{#1} \fi \left[#2\right]}}
\newcommand{\var}[2]{{\ifx&#1& \Var \else \Var_{#1}\fi \left[#2\right]}}
\newcommand{\MAD}{\mathop{\mathrm{mad}}}
\newcommand{\sdv}{\mathop{\mathrm{sd}}}
\newcommand{\veps}{\varepsilon}

\newcommand{\prob}[2]{\pr_{#1}\left[#2\right]}
\newcommand{\nope}[1]{}

\iffull
\newtheorem{remark}[thm]{Remark}
\fi

\newcommand{\interact}[2]{#1 {\rightarrow \atop \leftarrow} #2}

\iffalse
\newcommand{\vnote}[1]{\textcolor{red}{{\bf (Vitaly:} {#1}{\bf ) }} }
\newcommand{\tnote}[1]{\textcolor{blue}{{\bf (Thomas:} {#1}{\bf ) }}}
\else
\newcommand{\vnote}[1]{}
\newcommand{\tnote}[1]{}
\fi

\newcommand{\cump}[2]{\mathrm{cdf}_{#1}(#2)}

\newcommand{\iqr}[2]{\mathrm{qi}_{#1}\left(#2\right)}
\newcommand{\ind}{\mathbbm{1}}
\newcommand{\median}{\mathrm{median}}

\providecommand\X{\mathcal{X}}
\providecommand\Z{\mathcal{Z}}
\providecommand{\D}{{\mathcal D}}
\providecommand{\F}{{\mathcal F}}
\providecommand{\cP}{{\mathcal P}}
\providecommand{\cQ}{{\mathcal Q}}

\title{Generalization for Adaptively-chosen Estimators via Stable Median}
\iffull
\author{Vitaly Feldman\\\texttt{vitaly@post.harvard.edu} \and Thomas Steinke\\\texttt{median@thomas-steinke.net}
}
\date{IBM Research -- Almaden}
\else
\coltauthor{\Name{Vitaly Feldman} \addr IBM Research -- Almaden  \Email{vitaly@post.harvard.edu}
\AND \Name{Thomas Steinke}  \addr IBM Research -- Almaden \Email{median@thomas-steinke.net}
}
\fi

\begin{document}
\maketitle

\begin{abstract}
Datasets are often reused to perform multiple statistical analyses in an adaptive way, in which each analysis may depend on the outcomes of previous analyses on the same dataset. Standard statistical guarantees do not account for these dependencies and little is known about how to provably avoid overfitting and false discovery in the adaptive setting. We consider a natural formalization of this problem in which the goal is to design an algorithm that, given a limited number of i.i.d.~samples from an unknown distribution, can answer adaptively-chosen queries about that distribution.

We present an algorithm that estimates the expectations of $k$ arbitrary adaptively-chosen real-valued estimators using a number of samples that scales as $\sqrt{k}$. The answers given by our algorithm are essentially as accurate as if fresh samples were used to evaluate each estimator. In contrast, prior work yields error guarantees that scale with the worst-case sensitivity of each estimator.
We also give a version of our algorithm that can be used to verify answers to such queries where the sample complexity depends logarithmically on the number of queries $k$ (as in the reusable holdout technique).

Our algorithm is based on a simple approximate median algorithm that satisfies the strong stability guarantees of differential privacy. Our techniques provide a new approach for analyzing the generalization guarantees of differentially private algorithms.
\end{abstract}

\input{median_intro}

\input{prelims}

\input{approx_median}

\input{generalization}

\input{extensions}

\remove{
\section{Discussion}
The existing upper \citep{DworkFHPRR15,BassilyNSSSU16} and lower \citep{HardtU14,SteinkeU15} bounds are essentially matching insofar as the number of queries $k$. However, there remains a linear gap in terms of the accuracy parameter $\tau$. (The lower bound can be made to scale as $n = \Omega(\sqrt{k}/\tau)$ by a black-box reduction to the $\tau=0.1$ case, whereas the upper bound is $n=\tilde{O}(\sqrt{k}/\tau^2)$.)
This paper conducts a closer examination of the accuracy guarantees provided in the upper bounds.
}

\iffull
\printbibliography
\else
\bibliography{refs,vf-allrefs}
\fi

\appendix
\input{mad-appendix}

\end{document}

%% file: median_intro.tex
\section{Introduction}

Modern data analysis is an iterative process in which multiple algorithms are run on the same data, often in an adaptive fashion: the algorithms used in each step depend on the results from previous steps. Human decisions may also introduce additional, implicit dependencies between the data and the algorithms being run. In contrast, theoretical analysis in machine learning and statistical inference focuses on fixed and pre-specified algorithms run on ``fresh" data. The misuse of theoretical guarantees given by such analyses in real-world data analysis can easily lead to overfitting and false discovery.

While the issue has been widely recognized and lamented for decades, until recently the only known ``safe" approach for dealing with  general dependencies that can arise in the adaptive setting was ``data splitting'' --- that is, dividing up the data so that untouched data is available when needed. While easy to analyze, data splitting is overly conservative for most settings and prohibitively expensive for analyses that include multiple adaptive steps. In a recent work, Dwork, Feldman, Hardt, Pitassi, Reingold, and Roth \citeyearpar{DworkFHPRR14:arxiv} have proposed a new formalization of this problem. In their formulation, the analyst (or data analysis algorithm) does not have direct access to the dataset and instead can
ask queries about the (unknown) data distribution. Each query corresponds to a procedure that the analyst wishes to execute on the data. The challenge is thus to design an algorithm that provides answers to queries that are close to answers that would have been obtained had the corresponding analyses been run on samples freshly drawn from the data distribution.


\subsubsection*{Previous Work}
\citet{DworkFHPRR14:arxiv} consider {\em statistical queries} \citep{Kearns:98}. Each statistical query is specified by a function $\psi:\X \rar [-1,1]$ and corresponds to analyst wishing to compute
$\ex{X\sim \cP}{\psi(X)}$, where $\cP$ is the data distribution (usually done by using the empirical mean of $\psi$ on the dataset). A value $v \in \R$ is a valid response to such query if $|v - \ex{X\sim \cP}{\psi(X)}| \leq \tau$, where the parameter $\tau$ determines the desired accuracy. Dwork \etal demonstrated that, if a query answering algorithm $M : \X^n \to [-1,1]$ is \emph{differentially private}\footnote{Differential privacy is a strong notion of algorithmic stability developed in the context of privacy preserving data analysis \citep{DworkMNS:06}. See Definition \ref{def:DP}.} as well as empirically accurate --- that is, $\left|M(S) - \frac{1}{n} \sum_{i \in [n]} \psi(S_i) \right| \leq \frac{\tau}{2}$ with high probability --- then the answers produced by $M(S)$ will be valid with high probability if we take $S \sim \cP^n$ to be $n$ i.i.d.~samples from the data distribution. Using algorithms from the differential privacy literature, they demonstrated several improvements over sample splitting. For example, they showed that simply perturbing the empirical mean by adding Gaussian noise allows answering $k$ queries using a dataset whose size scales as $\sqrt{k}$. In contrast, either sample splitting or answering using empirical averages both require a dataset whose size scales linearly in $k$ \citep{DworkFHPRR14:arxiv}. (Note that if the queries were non-adaptive, $\log k$ scaling would be sufficient.) They also showed a computationally inefficient algorithm that has optimal $\log k$ dependence (at the expense of an additional $\sqrt{\log|\X|}$ factor, where $\X$ is the set of possible data points, i.e.~the potential support of $\cP$.). 

Bassily, Nissim, Smith, Steinke, Stemmer, and Ullman \citeyearpar{BassilyNSSSU16}  extended the results of Dwork \etal to \emph{low-sensitivity queries}, where a query is specified by a function $\phi : \X^n \to \R$ such that $|\phi(s)-\phi(s')|\leq 1/n$ if $s$ and $s'$ differ on a single element. Each such query corresponds to computing a real-valued estimator (or statistic) on the dataset. A value $v \in \R$ is a valid response to such query if $|v - \ex{S\sim \cP^n}{\phi(S)}| \leq \tau$. (Note that the empirical mean of a statistical query $\psi$ is a low-sensitivity query, namely $\phi(s)=\frac{1}{n} \sum_{i \in [n]} \psi(s_i)$.) They also introduced a simpler and quantitatively sharper analysis showing that answering $k$ statistical queries $\psi : \X \to [-1,1]$ to accuracy $\tau$ is possible efficiently with $n=\tilde{O}(\sqrt{k}/\tau^2)$ samples and inefficiently with $n = \tilde{O}(\sqrt{\log|\X|} \log(k)/\tau^3)$ samples --- in both cases improving the bounds of Dwork \etal~by a $\sqrt{\tau}$ factor.

\citet{HardtU14} and, subsequently with tighter parameters, \citet{SteinkeU15} showed 
that answering $k$ adaptively-chosen statistical queries (even to fixed accuracy, e.g.~$\tau=0.1$) requires a number of samples $n$ that scales with $\sqrt{k}$ in the worst case. This lower bound holds under one of two assumptions: Either the data universe is large -- that is, $|\X| \geq 2^k$ -- or $M$ is assumed to be computationally efficient so that it cannot break symmetric cryptography with $\log|\X|$-bit secret keys.

Still, in many applications the effects of adaptivity on generalization are much milder than in the adversarial setting considered in the lower bounds. For such settings, \citet{DworkFHPRR15:science} proposed the {\em reusable holdout} technique, which allows answering a large number of ``verification" queries. In this technique, the analyst has unresricted access to most of the dataset, but sets aside a subset of data as a ``holdout'' that is only accessed via queries. Given a query $\psi$ and an estimate $v$, the goal of the algorithm is to verify that $|v - \ex{X\sim \cP}{\psi(X)}| \lesssim \tau$. If the estimate is valid, then the algorithm only needs to respond ``Yes". Otherwise, it outputs ``No'' and a valid estimate. They demonstrated that it is possible to verify $k$ statistical queries while correcting at most $\ell$ of them using $n=\tilde{O}(\sqrt{\ell}\log(k)/\tau^2)$ samples.

Generalization guarantees in the adaptive setting can also be derived using related notions of stability such as stability in KL divergence \citep{BassilyNSSSU16,RaginskyRTWX16}. \citet{DworkFHPRR15:science} have  introduced an alternative approach to proving generalization guarantees in this setting that relies on bounding the amount of information between the choice of estimator and the dataset. For example, results obtained via differential privacy can also be derived using the notion of approximate max-information \citep{DworkFHPRR15:arxiv,RogersRST16}. Weaker forms of generalization guarantees can also be obtained via the standard mutual information \citep{RussoZ16,XuR17}.

\subsubsection*{Our Results}
In this work, we demonstrate a simple algorithm that can provide accurate answers to adaptively-chosen queries corresponding to arbitrary real-valued estimators. Specifically, let  $\phi : \X^t \to \R$ be an arbitrary estimator, where the expectation $\ex{Z \sim \cP^t}{\phi(Z)}$ is equal (in which case the estimator is referred to as unbiased) or sufficiently close to some parameter or value of interest. The number of samples $t$ that $\phi$ uses corresponds to the desired estimation accuracy (naturally, larger values of $t$ allow more accurate estimators). Our algorithm estimates the expectation $\ex{Z \sim \cP^t}{\phi(Z)}$ to within (roughly) the standard deviation of $\phi(Z)$ --- i.e.~$\tau \approx \sdv(\phi(\cP^t)) \doteq \sqrt{\var{Z \sim \cP^t}{\phi(Z)}}$ --- or, more generally, to within the interquartile range of the distribution of $\phi$ on fresh data (i.e.~the distribution of $\phi(Z)$ for $Z \sim \cP^t$, which we denote by $\phi(\cP^t)$). If $\phi(s)=\frac{1}{t} \sum_{i = 1}^t \psi(s_i)$ for a function $\psi : \X \to [-1,1]$, then the error roughly scales as $\tau \approx \sdv(\psi(\cP))/\sqrt{t}$. This gives a natural interpretation of $t$ as an accuracy parameter.

In contrast, given a comparable number of samples, the existing algorithms for statistical queries \citep{DworkFHPRR14:arxiv,BassilyNSSSU16} give an estimate with accuracy roughly $\tau \approx \sqrt{1/t}$ regardless of the variance of the query. This is not just an artifact of existing analysis since, to ensure the necessary level of differential privacy, this algorithm adds noise whose standard deviation scales as $\sqrt{1/t}$. For a statistical query $\psi : \X \to [-1,1]$, the standard deviation of $\psi(\cP)$ is upper bounded by $1$, but is often much smaller. For example, when estimating the accuracy of a binary classifier with (low) error $p$ our algorithm will give an estimate with accuracy that scales as $\sqrt{p/t}$, rather than $\sqrt{1/t}$. 


We now describe the guarantees of our algorithm more formally starting with the simple case of statistical queries. Given a statistical query $\psi:\X \rar [-1,1]$ and $t$ fresh random i.i.d.~samples $Z \in \X^t$ drawn from $\cP$, the empirical mean estimator $\phi(Z) \doteq \frac{1}{t} \sum_{i \in [t]} \psi(Z_i)$ is an unbiased estimator of $\ex{X \sim \cP}{\psi(X)}$ with standard deviation equal to $\sdv(\psi(\cP))/\sqrt{t}$. Applied to such estimators, our algorithm answers adaptively-chosen statistical queries with accuracy that scales as $\sdv(\psi(\cP))/\sqrt{t}$.
\begin{thm}\label{thm:intro-sq}
For all $\zeta>0$, $t\in\mathbb{N}$, $\beta>0$, $k\in\mathbb{N}$, and $n\geq n_0= O(t \sqrt{k\log(1/\beta)} \log(k/\beta\zeta))$, there exists an efficient algorithm $M$ that, given $n$ samples from an unknown distribution $\cP$, provides answers $v_1, \ldots, v_k \in [-1,1]$ to an adaptively-chosen sequence of queries $\psi_1, \ldots, \psi_k : \X \to [-1,1]$ and satisfies: $$\prob{S \sim \cP^n \atop M(S)}{\forall j \in [k] ~~~ \left|v_j-\ex{X \sim \cP}{\psi_j(X)}\right| \leq 2 \cdot \frac{\sdv(\psi_j(\cP))}{\sqrt{t}} + \zeta} \geq 1-\beta.$$
\end{thm}
Note that our guarantees also have an additional $\zeta$ {\em precision} term. The dependence of sample complexity on $1/\zeta$ is logarithmic and can be further improved using more involved algorithms. Thus the precision term $\zeta$, like the failure probability $\beta$, can be made negligibly small at little cost.  Previous work \citep{DworkFHPRR14:arxiv,BassilyNSSSU16} gives an error bound of $\sqrt{1/t}$, which is at least as large (up to constant factors) as our bound $\sdv(\psi_j(\cP))/\sqrt{t} + \zeta$, when $\zeta \leq 1/t$.
For comparison, in the non-adaptive setting the same error guarantee can be obtained using $n=O(t\log(k/\beta))$ samples (with $\zeta=1/t$).

Note that, for simplicity, in Theorem \ref{thm:intro-sq} the range of each query is normalized to $[-1,1]$. However this normalization affects only the precision term $\zeta$. In particular, for queries whose range is in an interval of length at most $b$, the number of samples that our result requires to get precision $\zeta$ scales logarithmically in $b/\zeta$. In contrast, the sample complexity of previous results scales quadratically in $b$.  Further, a more refined statement discussed below allows us to handle queries with arbitrary range.

For general real-valued estimators of the form $\phi:\X^t\rar [-1,1]$ our algorithm gives the following guarantees. (For simplicity we will assume that $t$ is fixed for all the queries.)
\begin{thm}\label{thm:intro-var}
For all $\zeta>0$, $t\in\mathbb{N}$, $\beta>0$, $k\in\mathbb{N}$, and $n\geq n_0=O(t \sqrt{k\log(1/\beta)} \log(k/\beta\zeta))$, there exists an efficient algorithm $M$ that, given $n$ samples from an unknown distribution $\cP$, provides answers $v_1, \ldots, v_k \in [-1,1]$ to an adaptively-chosen sequence of queries $\phi_1, \ldots, \phi_k : \X^t \to [-1,1]$ and satisfies: $$\prob{S \sim \cP^n \atop M(S)}{\forall j \in [k] ~~~ \left|v_j-\ex{Z \sim \cP^t}{\phi_j(Z)}\right| \leq 2\cdot \sdv(\phi_j(\cP^t)) +\zeta} \geq 1-\beta.$$
\end{thm}
 For general estimators, previous work gives accuracy guarantees in terms of the worst-case sensitivity. More formally, for $\phi:\X^n \rar \R$, let $\Delta(\phi)$ denote the worst-case sensitivity of $\phi$ --- that is, $\Delta(\phi) = \max_{z,z'} \phi(z)-\phi(z')$, where the maximum is over $z,z' \in \X^n$ that differ in a single element. 
The analysis of  \citet{BassilyNSSSU16} shows that for $k$ adaptively-chosen queries $\phi_1, \ldots, \phi_k$, each query $\phi_i$, can be answered with accuracy $\sqrt{n \cdot \sqrt{k}} \cdot \Delta(\phi_i)$ (ignoring logarithmic factors and the dependence on the confidence $\beta$). This setting is somewhat more general than ours, since each query is applied to the entire dataset available to the algorithm, whereas in our setting each query is applied to fixed-size subsamples of the dataset. This means that in this setting the space of estimators that can be applied to data is richer than in ours and might allow better estimation of some underlying quantity of interest. At the same time, our techniques gives better accuracy guarantees for finding the expectation of each estimator.

To see the difference between our setting and that in \citep{BassilyNSSSU16}, consider estimation of the lowest expected loss of a function from a family $\F$, namely\\$L^* \doteq \min_{f\in \F}\E_{X\sim \cP}[L(f,X)]$, where $L:\F\times \X \rar [-R,R]$ is some loss function. Given a dataset $s$ of size $n$, the standard ERM estimator is defined as $\phi_n(s) \doteq \min_{f\in \F}\fr{n} \sum_{i\in[n]} L(f,s_i)$. Using uniform convergence results, one can often obtain that $\left| L^* - \E_{S\sim \cP^n}[\phi_n(S)]\right| = O(d/\sqrt{n})$, for some $d$ that measures the capacity of $\F$ and also depends on $L$. The sensitivity of the estimator $\phi_n$ is upper bounded by $2R/n$. Thus the algorithm in \citep{BassilyNSSSU16} will give an estimate of $\E_{S\sim \cP^n}[\phi_n(S)]$ within (roughly) $R \cdot \sqrt{\frac{\sqrt{k}}{n}}$ and an upper bound on the total error will scale as $(d + R\cdot k^{1/4})/\sqrt{n}$. In our setting, the estimator $\phi_t$ will be used, where $t$ scales as $n/\sqrt{k}$. The bias of this estimator is $\left| L^* - \E_{S\sim \cP^t}[\phi_t(S)] \right| = O(d/\sqrt{t}) = O(d \cdot k^{1/4}/\sqrt{n})$. At the same time we can estimate $\E_{S\sim \cP^t}[\phi_t(S)]$ within (roughly) the standard deviation of $\phi_t$. The standard deviation of $\phi_t$ is always upper bounded by $2R \cdot \sqrt{t} = R \cdot k^{1/4}/\sqrt{n}$, but is often much smaller. Hence, depending on the setting of the parameters and the distribution $\cP$, our approach gives an error bound that is either higher by a factor of $k^{1/4}$ or lower by a factor of $R/d$ (than the approach in \citep{BassilyNSSSU16}). In other words, the two approaches provide incomparable guarantees for this problem.




Both Theorem \ref{thm:intro-var} and Theorem \ref{thm:intro-sq} are corollaries of the following more general result. Define $$\iqr{\D}{a,b} \doteq \{v \in \mathbb{R} ~:~ \pr_{Y \sim \D}[Y \leq v] > a ~\wedge~ \pr_{Y \sim \D}[Y < v] < b\}$$ to be the $(a,b)$-\emph{quantile interval} of the distribution $\D$. We refer to the $(1/4,3/4)$-quantile interval as the \emph{interquartile interval}.
\begin{thm}\label{thm:intro-iqr}
For all $T \subset \mathbb{R}$ with $|T|<\infty$, $t\in\mathbb{N}$, $\beta>0$, $k\in\mathbb{N}$, and $$n\geq n_0= O(t \sqrt{k\log(1/\beta)} \log(|T|k/\beta)),$$ there exists an efficient algorithm $M$ that, given $n$ samples from an unknown distribution $\cP$, provides answers $v_1, \ldots, v_k \in T$ to an adaptively-chosen sequence of queries $\phi_1, \ldots, \phi_k : \X^t \to T$ and satisfies: $$\prob{S \sim \cP^n \atop M(S)}{\forall j \in [k] ~~~ v_j \in \iqr{\phi_j(\cP^t)}{\frac{1}{4},\frac{3}{4}}} \geq 1-\beta.$$
\end{thm}
We make two remarks about the guarantee of the theorem:

\emph{Accuracy in terms of interquartile interval:} The accuracy guarantee of Theorem \ref{thm:intro-iqr} is that each returned answer lies in the $(1/4,3/4)$-quantile interval of the distribution of the query function on fresh data (i.e.~the distribution $\phi(\cP^t)$). The length of this interval is referred to as the {\em interquartile range}. This guarantee may appear strange at first sight, 
but it is actually a strengthening of Theorem \ref{thm:intro-var}: by Chebyshev's inequality, the interquartile interval of any distribution $\D$ lies within two standard deviations of the mean: $$\iqr{\D}{\frac{1}{4},\frac{3}{4}} \subseteq \left[\ex{Y \sim \D}{Y} - 2\cdot\sdv(\D) , \ex{Y \sim \D}{Y} + 2\cdot\sdv(\D) \right].$$
However, the interquartile interval may be significantly smaller if $\D$ is heavy-tailed. If, for example, the distribution $\D$ has infinite variance, then our guarantee is still useful, whereas bounds in terms of standard deviation will be meaningless.
This formulation does not even assume that the quantity of interest is the expectation of $\phi(\cP^t)$ (or even that this expectation exists). In fact, $\phi$ could be a biased estimator of some other parameter of interest.

Intuitively, we can interpret this accuracy guarantee as follows. If we know that a sample from $\phi(\cP^t)$ is an acceptable answer with probability at least $3/4$ and the set of acceptable answers forms an interval, then with high probability the answer returned by our algorithm is acceptable. The constants $1/4$ and $3/4$ are, of course, arbitrary. More generally we can demand $v_j \in \iqr{\phi_j(\cP^t)}{a,b}$ for any $0 \leq a < b \leq 1$. However, this reduction increases the sample complexity $n_0$ by a factor of $1/(b-a)^2$.

\emph{Finite range $T$:} Theorem \ref{thm:intro-iqr} assumes that the queries have a finite range. This is necessary for our algorithm, as the required number of samples grows with the size of $T$, albeit slowly.\footnote{Using a more involved algorithm from \citet{BunNSV15}, it is possible to improve the dependence on the size of $T$ from $O(\log |T|)$ to $2^{O(\log^* |T|)}$, where $\log^*$ denotes the iterated logarithm --- an extremely slow-growing function.} To obtain Theorem \ref{thm:intro-var} from Theorem \ref{thm:intro-iqr}, we simply set $T$ to be the discretization of $[-1,1]$ with granularity $\zeta$ and round the output of $\phi : \X^t \to [-1,1]$ to the nearest point in $T$; this introduces the additive error $\zeta$. However, allowing $T$ to be arbitrary provides further flexibility. For example, $T$ could be a grid on a logarithmic scale, yielding a multiplicative, rather than additive, accuracy guarantee. Furthermore, in some settings the range of the query is naturally finite and the scale-free guarantee of Theorem \ref{thm:intro-iqr} comes at no additional cost. We also note that, in general, our result allows a different $T_j$ to be chosen (adaptively) for each query $\phi_j$ as long as the size of each $T_j$ is upper-bounded by some fixed value.

Another advantage of this formulation is that it removes the dependence on the entire range of $\phi$: If we know that the interquartile interval of $\phi(\cP^t)$ lies in some interval $[a,b]$, we can truncate the output range of $\phi$ to $[a,b]$ (and then discretize if necessary). This operation does not affect the interquartile interval of $\phi(\cP^t)$ and hence does not affect the guarantees of our algorithm. In particular, this means that in Theorem \ref{thm:intro-var} we do not need to assume that the range of each $\phi$ is bounded in $[-1,1]$; it is sufficient to assume that the interquartile interval of $\phi(\cP^t)$ lies in $[-1,1]$ to obtain the same guarantee. For example, if we know beforehand that the mean of $\phi(\cP^t)$ lies in $[-1,1]$ and its standard deviation is at most $1$ then we can truncate the range of $\phi$ to $[-3,3]$.

\paragraph{Verification queries:}
We next consider queries that ask for ``verification'' of a given estimate of the expectation of a given estimator  --- each query is specified by a function $\phi:\X^t \rar \R$ and a value (or ``guess'') $v \in \R$. The task of our algorithm is to check whether or not $v \in \iqr{\phi(\cP^t)}{\rho,1-\rho}$ for some $\rho$ chosen in advance.
Such queries are used in the reusable holdout setting \citep{DworkFHPRR15:science} and in the {\tt EffectiveRounds} algorithm that uses fresh subset of samples when a verification query fails    \citep{DworkFHPRR14:arxiv}. We give an algorithm for answering adaptively-chosen verification queries with the following guarantees.
\begin{thm} \label{thm:intro-rh}
For all $\alpha,\beta,\rho \in (0,1/4)$ with $\alpha<\rho$ and $t,\ell,k,n \in \mathbb{N}$ with $$n \geq n_0= O(t\sqrt{\ell\log(1/\alpha\beta)}\log(k/\beta)\rho/\alpha^2),$$ there exists an efficient algorithm $M$ that, given $n$ samples from an unknown distribution $\cP$, provides answers to an adaptively-chosen sequence of queries $(\phi_1,v_1), \ldots, (\phi_k,v_k)$ (where $\phi_j : \X^t \to \R$ and $v_j \in \R$ for all $j \in [k]$) and satisfies the following with probability at least $1-\beta$: for all $j \in [k]$
\begin{itemize}
\item If $v_j \in \iqr{\phi_j(\cP^t)}{\rho,1-\rho}$, then the algorithm outputs ``Yes''.
\item If $v_j \notin \iqr{\phi_j(\cP^t)}{\rho-\alpha,1-\rho+\alpha}$ the algorithm outputs ``No''.\footnote{If neither of these two cases applies, the algorithm may output either ``Yes'' or ``No.''}
\end{itemize}
However, once the algorithm outputs ``No'' in response to $\ell$ queries, it stops providing answers.
\end{thm}
To answer the $\ell$ queries that do not pass the verification step we can use our algorithm for answering queries in Theorem \ref{thm:intro-iqr} (with $k$ there set to $\ell$ here).

\paragraph{Answering many queries:}
Finally, we show an (inefficient) algorithm that requires a dataset whose size scales as $\log k$ at the expense of an additional $\sqrt{t \log|\X|}$ factor.

\begin{thm}\label{thm:intro-iqr-mwu}
For all $T \subset \mathbb{R}$ with $|T|<\infty$, $t\in\mathbb{N}$, $\beta>0$, $k\in\mathbb{N}$, and $$n\geq n_0=O(t^{3/2} \cdot \sqrt{\log|\X| \log(1/\beta)} \cdot  \log(k\log|T|/\beta)),$$ there exists an algorithm $M$ that, given $n$ samples from an unknown distribution $\cP$ supported on $\X$, provides answers $v_1, \ldots, v_k \in T$ to an adaptively-chosen sequence of queries $\phi_1, \ldots, \phi_k : \X^t \to T$ and satisfies $$\prob{S \sim \cP^n \atop M(S)}{\forall j \in [k] ~~~ v_j \in \iqr{\phi_j(\cP^t)}{\frac{1}{4},\frac{3}{4}}} \geq 1-\beta.$$
\end{thm}

We remark that when applied to low-sensitivity queries with $t = 1/\tau^2$, this algorithm improves dependence on $|\X|$ and $\tau$ from $n=\tilde{O}(\log(|\X|)/\tau^4)$ in \citet{BassilyNSSSU16} to $n=\tilde{O}(\sqrt{\log|\X|}/\tau^3)$ (although, as pointed out above, the setting in that work is not always comparable to ours).

\subsubsection*{Techniques}
Like the prior work \citep{DworkFHPRR14:arxiv,BassilyNSSSU16}, we rely on properties of differential privacy \citep{DworkMNS:06}. Differential privacy is a stability property of an algorithm, namely it requires that replacing any element in the input dataset results in a small change in the output distribution of the algorithm. 
As a result, a function output by a differentially private algorithm on a given dataset generalizes to the underlying distribution \citep{DworkFHPRR14:arxiv,BassilyNSSSU16}.
Specifically, if a differentially private algorithm is run on a dataset drawn i.i.d~from any distribution and the algorithm outputs a low-sensitivity function, then the empirical value of that function on the input dataset is close to the expectation of that function on a fresh dataset drawn from the same distribution.

The second crucial property of differential privacy is that it composes adaptively: running several differentially private algorithms on the same dataset still satisfies differential privacy (with somewhat worse parameters) even if each algorithm depends on the output of all the previous algorithms. This property makes it possible to answer adaptively-chosen queries with differential privacy and a number of algorithms have been developed for answering different types of queries. The generalization property of differential privacy then implies that such algorithms can be used to provide answers to adaptively-chosen queries while ensuring generalization \citep{DworkFHPRR14:arxiv}.



For each query, the algorithm of Theorem \ref{thm:intro-iqr} first splits its input, consisting of $n$ samples, into $m=n/t$ disjoint subsamples of size $t$ and computes the estimator $\phi: \X^t \rar \R$ on each. It then outputs an approximate \emph{median} of the resulting values in a differentially private manner.  Here an approximate median is any value that falling between the $(1-\alpha)/2$-quantile and the $(1+\alpha)/2$-quantile of the $m$ computed values (for some approximation parameter $\alpha$). In addition to making the resulting estimator more stable, this step also amplifies the probability of success of the estimator. 
For comparison, the previous algorithms compute the estimator once on the whole sample and then output the value in a differentially private manner.


We show that differential privacy ensures that an approximate empirical median with high probability falls within the true interquartile interval of the estimator on the input distribution. Here we rely on the known strong connection between differential privacy and generalization \citep{DworkFHPRR14:arxiv,BassilyNSSSU16}. However, our application relies on stability to replacement of any one of the $m$ subsamples (consisting of $t$ points) used to evaluate the estimator, whereas previous analyses used the stability under replacement of any one out of the $n$ data points. The use of this stronger condition is crucial for bypassing the limitations of previous techniques while achieving improved accuracy guarantees.


A number of differentially-private algorithms for approximating the empirical median of values in a dataset have been studied in the literature. One common approach to this problem is the use of local sensitivity \citep{NissimRS07} (see also \citet{DworkLei:09}). This approach focuses on additive approximation guarantees and requires stronger assumptions on the data distribution to obtain explicit bounds on the approximation error.

In this paper we rely on a data-dependent notion of approximation to the median in which the goal is to output any value $v$ between the $(1-\alpha)/2$-quantile and the $(1+\alpha)/2$-quantile of the empirical distribution for some approximation parameter $\alpha$. It is easy to see that this version is essentially equivalent to the {\em interior point} problem in which the goal is to output a value between the smallest and largest values in the dataset.  \citet{BunNSV15} recently showed that the optimal sample complexity of privately finding an interior point in a range of values $T$ lies between $m=2^{(1+o(1))\log^* |T|}$ and $m=\Omega(\log^* |T|)$, where $\log^*$ is iterated logarithm or inverse tower function satisfying $\log^*(2^x)=1+\log^*(x)$.

The algorithm in \citep{BunNSV15} is relatively complex and, therefore, here we will use a simple and efficient algorithm that is based on the exponential mechanism \citep{McSherryTalwar:07}, and is similar to an algorithm by \citet{Smith11} to estimate quantiles of a distribution. This algorithm has sample complexity $m=O(\log |T|)$ (for constant $\alpha > 0$). In addition, for our algorithm that answers many queries we will use another simple algorithm based on approximate binary search, which can yield sample complexity $m=O(\sqrt{\log |T|})$; it has the advantage that it reduces the problem to a sequence of statistical queries.\footnote{We are not aware of a formal reference but it is known that the exponential mechanism and binary search can be used to privately find an approximate median in the sense that we use in this work (\eg \citep{RaskhodnikovaSmith:10course,Nissim:14course}).}


Now the second part of our proof: we show that an approximate empirical median is also in the interquartile interval of the distribution. i.e.~we show that any value in the \emph{empirical} (3/8,5/8)-quantile interval falls in the \emph{distribution's} $(1/4,3/4)$-quantile interval. A value $v \in T$ is an approximate empirical median if the empirical cumulative distribution function at $v$ is close to $1/2$ --- that is, $\cump{S}{v}\doteq \pr_{Y\sim S}[Y \leq v] \approx 1/2$, where $S \sim \D^n$ is the random samples and $Y\sim S$ denotes picking a sample from $S$ randomly and uniformly.
Note that $\cump{S}{v}$ is the empirical mean of a statistical query over $S$, whereas the true cumulative distribution function $\cump{\D}{v}\doteq \pr_{Y\sim \D}[Y \leq v]$ is the true mean of the same statistical query.
Hence we can apply the known strong connection between differential privacy and generalization for statistical queries \citep{DworkFHPRR14:arxiv,BassilyNSSSU16} to obtain that $\cump{\D}{v} \approx \cump{S}{v}$. This ensures $\cump{\D}{v} \approx 1/2$ and hence $v$ falls in $\iqr{\D}{1/4,3/4}$ (with high probability).


For estimators that produce an accurate response with high probability (such as any well-concentrated estimator) we give a different, substantially simpler way to prove high probability bounds on the accuracy of the whole adaptive procedure. This allows us to bypass the known proofs of high-probability bounds that rely on relatively involved arguments \citep{DworkFHPRR14:arxiv,BassilyNSSSU16,RogersRST16}.

Our algorithm for answering verification queries (Theorem \ref{thm:intro-rh}), is obtained by reducing the verification step to verification of two statistical queries for which the domain of the function is $\X^t$ and the expectation is estimated relative to $\cP^t$. To further improve the sample complexity, we observe that it is possible to calibrate the algorithm for verifying statistical queries from \citep{DworkFHPRR15:science} to introduce less noise when $\rho$ is small. This improvement relies on sharper analysis of the known generalization properties of differential privacy that has already found some additional uses \citep{SteinkeU17,NissimS17} (see Section \ref{sec:mad-generalization} for details).

Our algorithm for answering many queries (Theorem \ref{thm:intro-iqr-mwu}) is also obtained by a reduction to statistical queries over $\X^t$. In this case we use statistical queries to find a value in the interquartile interval of the estimator via a binary search.

Finally, we remark that the differentially private algorithms that we use to answer queries might also be of interest for applications in private data analysis. These algorithms demonstrate that meaningful privacy and error guarantees can be achieved without any assumptions on the sensitivity of estimators. From this point of view, our approach is an instance of subsample-and-aggregate technique of \citet{NissimRS07}, where the approximate median algorithm is used for aggregation. We note that \citet{Smith11} used a somewhat related approach that also takes advantage of the concentration of the estimator around its mean during the aggregation step. His algorithm first clips the tails of the estimator's distribution via a differentially private quantile estimation  and then uses simple averaging with Laplace noise addition as the aggregation step. His analysis is specialized to estimators that are approximately normally distributed and hence the results are not directly comparable with our more general setting.






%% file: prelims.tex
\section{Preliminaries}
For $k \in \mathbb{N}$, we denote $[k] = \{1,2,\ldots, k\}$ and we use $a_{[k]}$ as a shorthand for the $k$-tuple $(a_1,a_2,\ldots,a_k)$.
For a condition $E$ we use $\ind(E)$ to denote the indicator function of $E$. Thus $\ex{}{\ind(E)}=\prob{}{E}$. 

For a randomized algorithm $M$ we use $Y\sim M$ to denote that 
$Y$ is random variable obtained by running $M$. For a distribution $\D$ we use $Y \sim \D$ to denote that $Y$ is a random variable distributed according to $\D$.
For $0\leq \alpha<\beta\leq 1$ and a distribution $\D$, we denote the $(\alpha,\beta)$-quantile interval of $\D$ by $$\iqr{\D}{\alpha,\beta} \doteq \left\{v \in \mathbb{R} ~:~ \pr_{Y \sim \D}[Y \leq v] > \alpha ~\wedge~ \pr_{Y \sim \D}[Y < v] < \beta \right\}.$$
We refer to $\iqr{\D}{1/4,3/4}$ as the interquartile interval of $\D$ and the length of this interval is the interquartile range.
For $s\in \R^n$ we denote by $\iqr{s}{\alpha,\beta}$ the empirical version of the quantity above, that is the $(\alpha,\beta)$-quantile interval obtained by taking $\D$ to be the uniform distribution over the elements of $s$. In general, we view datasets as being equivalent to a distribution, namely the uniform distribution on elements of that dataset.


For our algorithms, a query is specified by a function $\phi : \X^t \to \R$. For notational simplicity we will often set $\Z=\X^t$ and partition a dataset $s \in \X^{n}$ into $m$ points $s_1, \ldots, s_m \in \Z$. Therefore throughout our discussion we will have $n=mt$. For $s \in \Z^m$, we define $\phi(s) = (\phi(s_1), \ldots, \phi(s_m))$ to be the transformed dataset. Similarly, for a distribution $\cP$ on $\X$, we define $\D=\cP^t$ to be the corresponding distribution on $\Z$ and $\phi(\D)$ to be the distribution obtained by applying $\phi$ to a random sample from $\D$.
The expectations of these distributions are denoted $s[\phi] = \frac{1}{m} \sum_{i \in [m]} \phi(s_i)$ and $\D[\phi] = \ex{Z \sim \D}{\phi(Z)}$.

\subsection{Adaptivity}

A central topic in this paper is the interaction between two algorithms $A$ --- the analyst (who might even be  adversarial) --- and $M$ --- our query-answering algorithm.  Let $\cQ$ be the space of all possible queries and $\A$ be the set of all possible answers. In our applications $\cQ$ will be a set of functions $\phi : \Z \to \R$, possibly with some additional parameters and $\A$ will be (a subset of) $\R$. We now set up the notation for this interaction.

\begin{figure}[h!]
\begin{framed}
\begin{algorithmic}
\INDSTATE[0]{Input $s \in \Z^m$ is given to $M$.}
\INDSTATE[0]{For $j=1,2, \ldots, k$:}
\INDSTATE[1]{$A$ computes a query $q_j \in \cQ$ and passes it to $M$}
\INDSTATE[1]{$M$ produces answer $a_j \in \A$ and passes it to $A$}
\INDSTATE[0]{The output is the transcript $(q_1, q_2, \ldots, q_k, a_1, a_2, \ldots, a_k) \in \cQ^k \times \A^k$.}
\end{algorithmic}
\end{framed}
\vspace{-6mm}
\caption{$\interact{A}{M} : \Z^m \to \cQ^k \times \A^k$}\label{fig:interact}
\end{figure}

Given interactive algorithms $A$ and $M$, we define $\interact{A}{M}(s)$ to be function which produces a random transcript of the interaction between $A$ and $M$, where $s$ is the input given to $M$. Formally, Figure \ref{fig:interact} specifies how $\interact{A}{M} : \X^n \to \cQ^k \times \A^k$ is defined.

The transcript function $\interact{A}{M}$ provides a ``non-interactive view'' of the output of an interactive process. Our goal is thus to construct $M$ such that, for all $A$, the output of $\interact{A}{M}$ has the desired accuracy and stability properties.

\subsection{Differential Privacy}

We begin with the standard definition of differential privacy \citep{DworkMNS:06,DworkKMMN06}.

\begin{defn}[Differential Privacy]\label{def:DP}
An algorithm $M : \Z^m \to \mathcal Y$ is $(\varepsilon,\delta)$-differentially private if, for all datasets $s,s' \in \Z^m$ that differ on a single element, $$\forall E \subseteq \mathcal{Y} ~~~~~ \prob{}{M(s) \in E} \leq e^\varepsilon \prob{}{M(s')\in E} + \delta.$$
\end{defn}

Note that, throughout, we consider an algorithm $M : \X^{n} \to \mathcal{Y}$ and let $\Z=\X^t$ with $n=mt$ so that $M : \Z^m \to \mathcal{Y}$. Then we define differential privacy with respect to the latter view (that is, with respect to changing a whole tuple of $t$ elements in the original view of $M$). This is a stronger condition.

However, Definition \ref{def:DP} only covers non-interactive algorithms. Thus we extend it to interactive algorithms:

\begin{defn}[Interactive Differential Privacy]\label{def:DP2}
An interactive algorithm $M$ is $(\varepsilon,\delta)$-differentially private if, for all interactive algorithms $A$, the (non-interactive) algorithm $\interact{A}{M} : \Z^m \to \mathcal{Y}$ is $(\varepsilon,\delta)$-differentially private.
\end{defn}

We now record the basic properties of differential privacy. See the textbook of \citet{DworkRoth:14} for proofs and discussion of these results.

\begin{thm}[Postprocessing]\label{thm:post}
Let $M : \Z^m \to \mathcal{Y}$ be $(\varepsilon,\delta)$-differentially private. Let $F : \mathcal{Y} \to \mathcal{Y}'$ be an arbitrary randomized algorithm. Define $M' : \Z^m \to \mathcal{Y}'$ by $M'(s) = F(M(s))$. Then $M'$ is also $(\varepsilon,\delta)$-differentially private.
\end{thm}
Postprocessing is important as it allows us to perform further computation on the output of a differentially private algorithm without breaking the privacy guarantee.

We next state the key adaptive composition property of differential privacy, which bounds how rapidly differential privacy degrades under repeated use of the same dataset \citep{DworkRV10} (with sharper constants from \citep{BunS16}).
\begin{thm}[Adaptive Composition \citep{DworkRV10,BunS16}]\label{thm:composition}Fix $k \in \mathbb{N}$ and $\varepsilon_1,\ldots,\varepsilon_k,\delta_1,\ldots,\delta_k>0$. Let $M_1, \ldots, M_k : \mathcal{Z}^m \times \mathcal{Y} \to \mathcal{Y}$ be randomized algorithms.  Suppose that for all $j \in [k]$ and all fixed $y \in \mathcal{Y}$, the randomized algorithm $x \mapsto M_j(x,y)$ is $(\varepsilon_j,\delta_j)$-differentially private. Define $\hat M_1, \ldots, \hat M_k : \mathcal{Z}^m \to \mathcal{Y}$ inductively by $\hat M_1(x)=M_1(x,y_0)$ where $y_0 \in \mathcal{Y}$ is fixed and $\hat M_{j+1}(x) = M_{j+1}(x,\hat M_j(x))$ for $j \in [k-1]$. Then $\hat M_k$ is $(\hat \varepsilon,\hat \delta)$-differentially private for $$\hat\varepsilon = \frac12 \sum_{j \in [k]} \varepsilon_j^2 + \sqrt{2 \log(1/\delta')\sum_{j \in [k]} \varepsilon_j^2 } ~~~~\text{ and }~~~~ \hat\delta = \delta' + \sum_{j \in [k]} \delta_j, $$ where $\delta' \in (0,1)$ is arbitrary.\end{thm}

The analyst that asks queries can be seen as a postprocessing step on the output of a differentially private algorithm that answers the queries. Thus, by combining the adaptive composition and postprocessing properties of differential privacy we obtain that in order to ensure that an interactive algorithm is differentially private it is sufficient to ensure that each of the individual queries is answered with differential privacy.
\begin{thm}\label{thm:interactive-composition}
Fix $k \in \mathbb{N}$ and $\varepsilon,\delta>0$. Let $M : \Z^m \times \cQ \to \mathcal{A}$ be an algorithm, such that $M(s,q)$ provides an answer to the query $q$ using the dataset $s$ and $M$ is $(\varepsilon,\delta)$-differentially private for every fixed $s \in \Z^m$.

Define an interactive algorithm $M^{\otimes k}$ that takes as input $s\in \Z^m$ and answers $k$ adaptively-chosen queries $q_1, \ldots, q_k \in \mathcal{Q}$ where, for each $j \in [k]$, $M^{\otimes k}$ produces an answer by independently running $M(s,q_j)$.
Then $M^{\otimes k}$ is $\left(\frac12 k \varepsilon^2 + \varepsilon \sqrt{2k\ln(1/\delta')},\delta'+k\delta\right)$-differentially private for all $\delta' \in (0,1)$.
\end{thm}

\remove{
There is also a composition bound of $(k\varepsilon,k\delta)$-differential privacy. However, this is a weaker bound for our purposes.

Theorem \ref{thm:composition} was originally proved by Dwork, Rothblum, and Vadhan \citeyearpar{DworkRV10} (with slightly weaker parameters) and has subsequently been enhanced \citep{KairouzOV13,BunS16}. Furthermore, if rather than having the same privacy guarantee for each of the $k$ computations, we have that the $j^\text{th}$ computation is $(\varepsilon_j,\delta_j)$-differentially private, then the composed computation is $\left( \frac12 \sum_{j \in [k]} \varepsilon_j^2 + \sqrt{2\log(1/\delta') \sum_{j \in [k]} \varepsilon_j^2}, \delta'+\sum_{j \in [k]} \delta_j \right)$-differentially private.
}
\remove{ 
\subsection{Amplification by Median}
We state a useful and well-known lemma, which will be relevant later:
\begin{lem}[Amplification by Median]\label{lem:amp-med}
Let $X_1, \ldots, X_\ell$ be i.i.d.~real random variables drawn from distribution $\D$. Fix $\rho \in (0,1/2)$. Then $$\prob{}{\median(X_1, \ldots, X_\ell) \in \iqr{\D}{\rho,1-\rho}} \geq 1 - 2 \cdot e^{- 2 (1/2-\rho)^2 \ell} .$$
\end{lem}
In particular, setting $\rho=1/4$, we have that $$\iqr{\median(\D^\ell)}{2 \cdot e^{-\ell/8},1-2 \cdot e^{-\ell/8}} \subseteq \iqr{\D}{\frac 1 4 , \frac 3 4 } .$$
The key point is that there is nothing special about $1/4$ and $3/4$ in Theorem \ref{thm:intro-iqr}. We can replace them with numbers very close to $0$ and $1$ at very little cost. 
\begin{proof}
First note that $\iqr{\D}{\rho,1-\rho} = \iqr{\D}{\rho,1} \cap \iqr{\D}{0,1-\rho}$. Thus, by a union bound, it suffices to prove that \begin{equation}\prob{}{\median(X_1, \ldots, X_\ell) \notin \iqr{\D}{\rho,1}} \leq e^{- 2 (1/2-\rho)^2 \ell} \label{eqn:rho1}\end{equation} and \begin{equation}\prob{}{\median(X_1, \ldots, X_\ell) \notin \iqr{\D}{0,1-\rho}} \leq e^{- 2 (1/2-\rho)^2 \ell} \label{eqn:rho0}.\end{equation}
Define $$A = \sum_{i\in[\ell]} \mathbb{I}[X_i \notin \iqr{\D}{\rho,1}]$$ so that $$\prob{}{\median(X_1, \ldots, X_\ell) \notin \iqr{\D}{\rho,1}} \leq \prob{}{A \geq \ell/2}. $$
By definition, $\ex{}{\mathbb{I}[X_i \notin \iqr{\D}{\rho,1}]}\leq\rho$ for all $i \in [\ell]$. Thus $A$ is the sum of $\ell$ independent random variables supported on $\{0,1\}$ each with mean $\leq \rho$, whence, by Hoeffding's inequality, $$\prob{}{A \geq \ell/2} \leq \prob{}{A-\ex{}{A} \geq \ell/2-\rho\ell} \leq e^{-2(\ell/2-\rho\ell)^2/\ell} = e^{-2(1/2-\rho)^2 \ell},$$ as required to prove \eqref{eqn:rho1}. The proof of \eqref{eqn:rho0} is symmetric.
\end{proof}

}

%% file: approx_median.tex
\section{Approximate Median}
\newcommand{\amed}[2]{\iqr{#2}{\frac{1-#1}{2},\frac{1+#1}{2}}}
In this section, we present differentially private algorithms for outputting an approximate median of a real-valued dataset. Namely, for $s \in \R^m$, we define an $\alpha$-approximate median of $s$ to be any element of the set $\amed{\alpha}{s}$. In our application each real value is obtained by applying the given query function to a single subsample. 


Several differentially private algorithms for computing an approximate median are known. All of these algorithms assume that the input elements and the range of the algorithm are restricted to some finite set $T \subseteq \R$. The strongest upper bound was given in a recent work of \citet{BunNSV15}. They describe an $(\varepsilon,\delta)$-differentially private algorithm which, on input $s \in T^m$, outputs an $\alpha$-approximate median of $s$ as long as $m \geq (2+o(1))^{\log^* |T|} \cdot {O}(\log(1/\varepsilon\delta)/\varepsilon\alpha)$.\footnote{\citet{BunNSV15} consider the problem of outputting an interior point which is equivalent to our definition of a $1$-approximate median. However, by removing the elements of the dataset that are not in $\amed{\alpha}{s}$ $\alpha$-approximate median reduces to the interior point problem.}

\citet{BunNSV15} also prove a nearly tight lower bound of $m \geq \Omega(\log^* |T|)$ for $\alpha=\varepsilon=1$ and $\delta=1/100m^2$. This lower bound implies that privately outputting an approximate median is only possible if we restrict the data points to a finite range. We note that it is also known that for the stricter $\veps$-differential privacy the sample complexity of this problem for constant $\alpha > 0$ is $\theta(\log(|T|)/\veps)$ (\eg \cite{BunNSV15,BunSU17}).


Any differentially private algorithm for finding an approximate median can be used in our results. The algorithm in \citep{BunNSV15} is relatively involved and hence we will describe a simple algorithm for the problem that relies on a ``folklore" application of the exponential mechanism \citep{McSherryTalwar:07} (\eg \citep{Smith11,Nissim:14course}).

\begin{thm}\label{thm:EM-med}
For all $\varepsilon,\alpha,\beta>0$, finite $T \subset \R$, and all $m \geq 4 \ln(|T|/\beta)/\varepsilon\alpha$, there exists an $(\varepsilon,0)$-differentially private randomized algorithm $M$ that given a dataset $s \in \Z^m$ and a query $\phi : \Z \to T$ outputs an $\alpha$-approximate median of $\phi(s) \in T^m$ with probability at least $1-\beta$. The running time of the algorithm is $O(m \cdot \log |T|)$.
\end{thm}
\begin{proof}
The algorithm is an instantiation of the exponential mechanism \citep{McSherryTalwar:07} with the utility function $c:T\rar \R$ defined as $$c_{\phi(s)}(v) \doteq  \max \left\{  |\{i \in [m] : \phi(s_i) < v\}| , |\{i \in [m] : \phi(s_i) > v\}|\right\}. $$

The algorithm outputs each $v \in T$ with probability
$$\prob{}{M(s,\phi)=v} = \frac{\exp\left( \frac{-\varepsilon }{2} c_{\phi(s)}(v)\right)}
{\sum_{u \in T}\exp\left( \frac{-\varepsilon }{2} c_{\phi(s)}(u)\right) }.$$

Since $c_{\phi(s)}(v)$ has sensitivity $1$ as a function of $s$, this algorithm is $(\varepsilon,0)$-differentially private \cite[Theorem 3.10]{DworkRoth:14}. Moreover, we have the accuracy guarantee \cite[Corollary 3.12]{DworkRoth:14} \begin{equation}\forall s,\beta ~~~~~ \prob{V \sim M(s,\phi)}{c_{\phi(s)}(V) < \mathrm{OPT}_{\phi(s)} + \frac{2 \ln (|T|/\beta)}{\varepsilon} } \geq 1-\beta, \label{eqn:EMacc}\end{equation}
where $\mathrm{OPT}_{\phi(s)} \doteq \min_{u \in T} c_{\phi(s)}(u) \leq m/2$.
Assuming the event in \eqref{eqn:EMacc} happens for $V=v$ (that is, $c_{\phi(s)}(v) < \mathrm{OPT}_{\phi(s)} + \frac{2 \ln (|T|/\beta)}{\varepsilon}$), we have $$\prob{Y \sim \phi(s)}{Y \leq v} = 1-\frac{1}{m}  |\{i \in [m] : \phi(s_i) > v\}| > \frac12 - \frac{2\ln(|T|/\beta)}{\varepsilon m} \geq \frac{1-\alpha}{2},$$ as long as $\alpha \geq 4\ln(|T|/\beta)/\varepsilon m$, which is equivalent to $m \geq 4\ln(|T|/\beta)/\varepsilon\alpha$. Similarly, the event in \eqref{eqn:EMacc} implies that $$\prob{Y \sim \phi(s)}{Y < v}= \frac{1}{m}  |\{i \in [m] : \phi(s_i) < v\}| < \frac12 + \frac{2\ln(|T|/\beta)}{\varepsilon m}  \leq \frac{1+\alpha}{2}.$$
Thus $$\prob{V \sim M(s,\phi)}{V \in \amed{\alpha}{\phi(s)}} \geq 1-\beta$$ as long as $m \geq 4\ln(|T|/\beta)/\varepsilon\alpha$.

To get an upper bound on the running time we observe that using binary search, the elements of $T$ can be split into $m+1$ ``intervals" (that is contiguous subsets of $T$) with all elements of each interval having equal probability. This partition allows us to compute the normalization factor as well as the total probability of all the elements of $T$ in each interval in $O(m\log|T|)$ time. A random point from the desired distribution can now be produced by first picking the interval proportionally to its probability and then outputting a point in that interval randomly and uniformly. (We implicitly assume that the structure of $T$ is simple enough so that such operations can be performed in $O(\log |T|)$ time and ignore the time to evaluate $\phi$ on each of the elements of $s$.)
\end{proof}

We now describe another simple and ``folklore" algorithm (\eg \cite{RaskhodnikovaSmith:10course,Feldman:16sqvar}) for finding an approximate median of a distribution that reduces the problem to $O(\log |T|)$ statistical queries. Recall, that an $\alpha$-accurate response to a statistical query $\psi:\Z\rar [-1,1]$ relative to distribution $\D$ over $\Z$ is any value $v$ such that $|v - \D[\psi]|\leq \alpha$.
\begin{lem}
\label{lem:median2sq}
For all $\alpha>0$, finite $T \subset \R$, a query $\phi:\Z \to T$ and any distribution $\D$ over $T$, a value $v \in \amed{\alpha}{\phi(\D)}$ can be found using $\alpha/4$-accurate responses to at most $2\lceil\log_2 (|T|)\rceil$ (adaptively-chosen) statistical queries relative to distribution $\D$.
\end{lem}
\begin{proof}
Using binary search we find a point $v \in T$ that satisfies the conditions $p_\leq(v) > 1/2 - \alpha/4$ and $p_<(v) < 1/2 + \alpha/4$, where $p_\leq(v)$ (or $p_<(v)$) is the response to the statistical query $\psi(z) = \ind(v \leq \phi(z))$ ($\psi(z) = \ind(v < \phi(z))$, respectively). By the accuracy guarantees of the responses, we have that $|p_\leq(v) - \pr_{Z\sim \D}[\phi(Z) \leq v]| \leq \alpha/4$, and similarly for $p_<(v)$. We choose the next point to test depending on which of the conditions fails (note that we can assume that $p_<(v) \leq p_\leq(v)$ so at most one condition can fail). Further, for the true median point of $\phi(\D)$ (that is the point $v^*\in T$ for which $\prob{Z \sim \D}{\phi(Z) < v^*} < 1/2$ and $\prob{Z \sim \D}{\phi(Z) \leq v^*} \geq 1/2$) both conditions will be satisfied by the accuracy guarantees. Finally, by the accuracy guarantees, any point $v'$ that satisfies both of these conditions is an $\alpha$-approximate median for $\phi(\D)$.
\end{proof}

To find the $\alpha$-approximate median of values $\phi(s)\in T^m$ this reduction needs to be applied to the uniform distribution over the elements of $\phi(s)$. Answering statistical queries relative to this empirical distribution (commonly referred to as linear or counting queries) with differential privacy is a well-studied problem. For example, by using the standard Laplace or Gaussian noise addition algorithm one can obtain the following algorithm for finding an $\alpha$-approximate median (see \citep{BunS16} for an analysis of the privacy properties of Gaussian noise addition).
\begin{cor}
\label{thm:SQ-med}
For all $\varepsilon,\delta,\alpha,\beta\in (0,1/2)$, finite $T \subset \R$ and all $$m \geq \frac{12\sqrt{2\lceil \log_2 |T| \rceil \cdot \ln (1/\delta) \cdot \ln (2\lceil \log_2 |T| \rceil/\beta)}}{\varepsilon\alpha} = O\left(\frac{\sqrt{\log |T| \cdot \log(1/\delta) \cdot \log\left(\frac{\log|T|}{\beta}\right)}}{\varepsilon \alpha}\right),$$ there exists an $(\varepsilon,\delta)$-differentially private randomized algorithm $M$ that given $s \in \Z^m$ and $\phi:\Z \to  T$ outputs an $\alpha$-approximate median of $\phi(s)$ with probability at least $1-\beta$. The running time of the algorithm is $O(m \cdot \log |T|)$.
\end{cor}

\remove{
In particular, by letting $T$ be a uniform grid of size $1/\Delta$ and rounding inputs, we can modify the algorithm in Theorem \ref{thm:EM-med} to obtain the following result.
\begin{cor}
For all $\varepsilon,\alpha,\beta,\Delta>0$ and all $m  > 4 \log(1/\Delta\beta)/\varepsilon\alpha$, there exists a $(\varepsilon,0)$-differentially private randomized algorithm $M : [0,1]^m \to  [0,1]$ whose output is $\Delta$-close to an $\alpha$-approximate median with probability at least $1-\beta$ --- that is, $$\forall S \in [0,1]^m ~~~~~ \prob{v \sim M(S)}{[v-\Delta,v+\Delta] \cap \amed{\alpha}{S} \ne \emptyset} \geq 1-\beta.$$
\end{cor}

Combining Theorem \ref{thm:EM-med} with composition of differential privacy (Theorem \ref{thm:composition}) yields the following.

\begin{thm}\label{thm:em-adapt}
Fix $\alpha,\beta,\delta \in (0,1)$ and $m,k \in \mathbb{N}$.
There exists a $(\varepsilon,\delta)$-differentially private algorithm $M$ that takes as input $m$ i.i.d.~samples $S \in \Z^m$ from an unknown population $\D$ and answers adaptively-chosen queries $\phi_1, \ldots, \phi_k : \Z \to T$ with $v_1, \ldots, v_k \in T$ such that $$\forall S \in \Z^m ~~~~~ \prob{(\phi_{[k]},v_{[k]}) \sim \interact{A}{M}(S)}{\forall j \in [k] ~~~ v_j \in \amed{\alpha}{\phi_j(S)}} \geq 1-\beta,$$ where $\phi_j(S)$ denotes the uniform distribution on the multiset $\{\phi_j(S_i) : i \in [m]\}$ and $$\varepsilon = \frac{k}{2} \left( \frac{4\log(|T|/\beta)}{\alpha m} \right)^2 + \frac{4\log(|T|/\beta)}{\alpha m} \sqrt{2k\log(1/\delta)}.$$
\end{thm}
} 

%% file: generalization.tex
\section{Generalization from Differential Privacy}

In this section we provide two proofs that differential privacy gives generalization guarantees for statistical queries. The first proof --- which we call strong generalization --- is most similar to previous work, whereas the second proof --- which we call simple generalization --- is much simpler, but gives a weaker bound that is only suitable for estimators that are well-concentrated.

\subsection{Strong Generalization}
\label{sec:mad-generalization}
Theorem \ref{thm:gen} in this subsection shows that any differentially private algorithm generalizes with high probability, with a small blowup in the allowed generalization error. This proof closely follows that of \citet{BassilyNSSSU16}, but is quantitatively sharper. This quantitative sharpening allows us to estimate the probability that a given value $v$ is outside of $\iqr{\phi(\D)}{\rho,1-\rho}$ with higher accuracy (that scales with $\rho$ as in the non-adaptive case). We use this sharper version in Section \ref{sec:reusable}; however, for results in this section $\rho =1/4$ and therefore the bounds from \citep{BassilyNSSSU16} suffice.

The \emph{mean absolute deviation} of a distribution $\D$ over $\R$ is defined as
\equ{\MAD(\D) \doteq \E_{Y' \sim \D}\left[\left|Y'-\E_{Y \sim \D}[Y]\right|\right].\label{eq:def-mad}}

\begin{thm}\label{thm:gen}
Fix $\alpha,\beta,\gamma \in (0,1)$ and $m,k \in \mathbb{N}$. Set $\varepsilon = \frac12\ln(1+\gamma)$ and $\delta=\alpha\beta/16$. Suppose $m \geq \frac{8}{\varepsilon\alpha} \ln(2k/\beta)$.
Let $M : \Z^m \to \mathcal{F}_{[0,1]}^k$ be a $(\varepsilon,\delta)$-differentially private algorithm with $\mathcal{F}_{[0,1]}$ being the set of functions $\phi : \Z \to [0,1]$. Let $\D$ be a distribution on $\Z$. Then $$\prob{S \sim \D^m \atop \phi_{[k]} \sim M(S)}{\forall j \in [k] ~~~~ S[\phi_j]-\D[\phi_j] \leq \alpha + \gamma \cdot \MAD(\phi_j(\D))} \geq 1-\beta.$$
\end{thm}
Theorem \ref{thm:gen} is proved in Appendix \ref{app:mad-proof}. 

Note that, by Jensen's inequality, $\MAD(\phi(\D)) \leq \sdv(\phi(\D)) $. Thus Theorem \ref{thm:gen} gives an error bound that scales with the standard deviation of the query (plus the absolute $\alpha$ term). Also, by the triangle inequality and the fact that $\phi(z) \geq 0$ for all $z$, it holds that $$\MAD(\phi(\D)) \leq 2 \cdot \D[\phi].$$ Thus Theorem \ref{thm:gen} can also be interpreted as giving a multiplicative accuracy guarantee (plus the additive $\alpha$).
In comparison, the bound of \citet{BassilyNSSSU16} can be obtained (up to constants) by substituting the upper bound $\MAD(\phi(\D)) \leq 1$ into Theorem \ref{thm:gen}. Thus, when $\MAD(\phi(\D)) \ll 1$, our bound is sharper.

As stated, Theorem \ref{thm:gen} only applies in the non-adaptive setting and to statistical queries. However, we can easily extend it using the monitor technique of \citet{BassilyNSSSU16} and the cumulative probability function:

\begin{thm}\label{thm:gen-adapt}
Fix $\beta \in (0,1)$ and $k,m \in \mathbb{N}$ with $m \geq 2560 \ln(2k/\beta)$.
Let $M$ be an $(1/20,\beta/256)$-differentially private interactive algorithm that takes as input $s \in \Z^m$ and provides answers $v_1, \ldots, v_k \in \R$ to an adaptively-chosen sequence of queries $\phi_1, \ldots, \phi_k : \Z \to \R$. Suppose that, for all $s \in \Z^m$ and all interactive algorithms $A$, \begin{equation}\prob{(\phi_{[k]},v_{[k]}) \sim \interact{A}{M}(s)}{\forall j \in [k] ~~~ v_j \in \iqr{\phi_j(s)}{\frac 3 8, \frac 5 8}} \geq 1 - \beta.\label{eqn:empacc}\end{equation}
Then, for all distributions $\D$ and all interactive algorithms $A$, $$\prob{S \sim \D^m \atop (\phi_{[k]},v_{[k]}) \sim \interact{A}{M}(S)}{\forall j \in [k] ~~~ v_j \in \iqr{\phi_j(\D)}{\frac 1 4, \frac 3 4}} \geq 1 - 2\beta.$$
\end{thm}
\begin{proof}
Let $\cQ$ be the set of functions $\phi : \Z \to \R$ and let $A$ be an arbitrary algorithm that asks queries in $\cQ$. We define $f : {\cQ}^k \times \R^k \to \mathcal{F}_{[0,1]}$ as follows, where $\mathcal{F}_{[0,1]}$ is the set of functions $\psi : \Z \to [0,1]$. Given $(\phi,v) \in \cQ^k \times \R^k$, define $\psi_1,\psi_{-1},\psi_2,\psi_{-2}, \ldots, \psi_k, \psi_{-k} : \Z \to \{0,1\}$ by $$\psi_j(x) \doteq \ind(\phi_j(x) \leq v_j) ~~~\text{ and }~~~ \psi_{-j}(x) \doteq \ind(\phi_j(x) \geq v_j)$$ and let $$f(\phi,v) \doteq \argmin_{\psi \in \{\psi_1,\psi_{-1},\psi_2,\psi_{-2}, \ldots, \psi_k, \psi_{-k}\}} \D[\psi].$$
By the postprocessing property of differential privacy (Theorem \ref{thm:post}), $f(\interact{A}{M}(s))$ is a $(\varepsilon,\delta)$-differentially private algorithm (relative to its input $s \in \Z^m$). Moreover, by our assumption \eqref{eqn:empacc}, $$\forall s \in \Z^m ~~~~~ \prob{\psi \sim f\left(\interact{A}{M}(s)\right)}{s[\psi] \geq \frac 3 8} \geq 1- \beta.$$
However, by Theorem \ref{thm:gen}, $$\prob{S \sim \D^m \atop \psi \sim f(\interact{A}{M}(S))}{S[\psi]-\D[\psi] \leq \frac{1}{8}} \geq \prob{S \sim \D^m \atop \psi \sim f(\interact{A}{M}(S))}{S[\psi]-\D[\psi] \leq \frac{1}{16} + \frac{1}{8} \cdot \MAD(\psi(\D))} \geq 1 - \beta.$$ Thus, by a union bound and the construction of $f$, $$\prob{S \sim \D^m \atop (\phi_{[k]},v_{[k]}) \sim \interact{A}{M}(S)}{\forall j \in [k] ~~~ v_j \in \iqr{\phi_j(\D)}{\frac 1 4 , \frac 3 4}} = \prob{S \sim \D^m \atop \psi \sim f(\interact{A}{M}(S))}{\D[\psi] \geq \frac 1 4} \geq 1- 2 \beta.$$
\end{proof}

Combining generalization (Theorem \ref{thm:gen-adapt}) with our approximate median algorithm (Theorem \ref{thm:EM-med}) and composition (Theorem \ref{thm:composition}) yields our main result, Theorem \ref{thm:intro-iqr}. We prove a somewhat more general statement that allows using different range $T_j$ for every query $\phi_j$. (The same generalization applies to all our other results, but we do not state it for brevity).
\begin{thm}
\label{thm:adaptive-iqr}
For any $\beta \in (0,1)$, $t,k,r \in \mathbb{N}$ and $\Z=\X^t$, and with $$n\geq n_0=O\left(t \sqrt{k\log(1/\beta)} \cdot \log(k r/\beta)\right)$$ there exists an interactive algorithm $M$ that takes as input a dataset $s\in \X^n$ and provides answers $v_1, \ldots, v_k \in \R$ to adaptively-chosen queries $(T_1,\phi_1), \ldots, (T_k,\phi_k)$, where for all $j\in [k]$, $|T_j| \leq r$ and $\phi_j:\X^t\rar T_j$ with the following accuracy guarantee. For all interactive algorithms $A$ and distributions $\cP$ on $\X$, $$\prob{S \sim \cP^n \atop \left(T_{[k]},\phi_{[k]},v_{[k]}\right) \sim \interact{A}{M}(S)}{\forall j \in [k] ~~~ v_j \in \iqr{\phi_j(\cP^t)}{\frac 1 4 ,\frac 3 4}} \geq 1-\beta. $$
\end{thm}
\begin{proof}
The algorithm $M$ promised by Theorem \ref{thm:adaptive-iqr} 
is described in Figure \ref{fig:median-strong}.
\begin{figure}[h!]
\begin{framed}
\begin{algorithmic}
\INDSTATE[0]{Input $S \in \X^{mt}$.}
\INDSTATE[0]{Partition $S$ into $S_1, \ldots, S_m \in \X^t$.}
\INDSTATE[0]{For $j=1,2,\ldots,k$:}
\INDSTATE[1]{Receive a set $T_j$ and a query $\phi_j : \X^t \to T_j$.}
\INDSTATE[1]{Run the $(\tilde\varepsilon,0)$-differentially private $1/4$-approximate median algorithm $\tilde M$ from Thm.~\ref{thm:EM-med} for $T_j$ and with inputs $(S_1, \ldots, S_m)$ and $\phi_j$ to obtain output $v_j \in T_j$.}
\INDSTATE[1]{Return answer $v_j$.}
\end{algorithmic}
\end{framed}
\vspace{-6mm}
\caption{Algorithm for answering adaptive queries.}\label{fig:median-strong}
\end{figure}

Let $\Z \doteq \X^t$ and assume that for some $r$ fixed in advance, $r \geq \max_{j\in [k]}|T_j|$. Theorem \ref{thm:EM-med} says that if $m \geq 4\ln(kr/\beta)/(\alpha \tilde\varepsilon)$, then each execution of the median algorithm is $(\tilde\veps,0)$-differentially private and outputs an $\alpha$-approximate median with probability at least $1-\beta/k$. Here $\alpha=1/4$, so this rearranges to $\tilde\varepsilon=16\ln(kr/\beta)/m$. This implies that for all interactive algorithms $A$ and every $s\in \Z^m$,
 \begin{equation}\prob{(T_{[k]},\phi_{[k]},v_{[k]}) \sim \interact{A}{M}(s)}{\forall j \in [k] ~~~ v_j \in \iqr{\phi_j(s)}{\frac 3 8, \frac 5 8}} \geq 1 - \beta .\label{eqn:emp-pf}\end{equation}
Interactive composition (Theorem \ref{thm:interactive-composition}) implies that $M$ is $(\varepsilon,\delta)$-differentially private for any $\delta \in (0,1)$ and \begin{equation}\varepsilon = \frac{k}{2} \left( \frac{16\ln(kr/\beta)}{m} \right)^2 + \frac{16\ln(kr/\beta)}{m} \sqrt{2k\ln(1/\delta)}.\label{eqn:comp-dp}\end{equation}

By Theorem \ref{thm:gen-adapt}, if, in addition to \eqref{eqn:emp-pf}, we have $\varepsilon \leq 1/20$ for $\delta=\beta/256$ in \eqref{eqn:comp-dp} and $m \geq 2560\ln(2k/\beta)$, then, for all distributions $\cP$, $\D \doteq \cP^t$ and all interactive algorithms $A$, $$\prob{S \sim \D^m \atop \left(T_{[k]},\phi_{[k]},v_{[k]}\right) \sim \interact{A}{M}(S)}{\forall j \in [k] ~~~ v_j \in \iqr{\phi_j(\D)}{\frac 1 4, \frac 3 4}} \geq 1 - 2\beta,$$
which is our desired conclusion.

It only remains to find the appropriate bound on the parameter $m$. We need $m \geq 2560\ln(2k/\beta)$ and $$\frac{1}{20} \geq \varepsilon = \frac{k}{2} \left( \frac{16\ln(kr/\beta)}{m} \right)^2 + \frac{16\ln(kr/\beta)}{m} \sqrt{2k\ln(256/\beta)}.$$
Setting $m = 640 \sqrt{\max\{k,16\} \cdot \ln(256/\beta)} \cdot \ln(kr/\beta)$ achieves this.
\end{proof}

We now state a simple corollary of Theorem \ref{thm:adaptive-iqr} that converts the $(1/4,3/4)$-quantile interval guarantees to explicit additive error guarantees. The error will be measured in terms of the mean absolute deviation of the query $\phi$ on inputs sampled from $\cP^t$ (eq.~\eqref{eq:def-mad}).
For normalization purposes we will also assume that queries are scaled by the analyst is such a way that both $\cP^t[\phi] \in [-1,1]$ and $\MAD(\phi(\cP^t))\leq 1$. Note that this assumption is implied by $\phi$ having range $[-1,1]$ and, in general, allows $\phi$ to have an infinite range.
\begin{cor}\label{cor:mad-range}
For $t\in\mathbb{N}$ and a distribution $\cP$ over $\X$, let $\F_{\cP,t}$ denote the set of functions $\phi:\X^t \rar \R$ such that $\cP^t[\phi] \in [-1,1]$ and $\MAD(\phi(\cP^t)) \leq 1$. For all $\zeta>0$, $\beta>0$, $k\in\mathbb{N}$, and $n\geq n_0=O\left(t \sqrt{k\log(1/\beta)} \cdot \log(k/(\zeta\beta))\right)$, there exists an efficient algorithm $M$ which takes a dataset $s\in \X^n$ as an input and provides answers $v_1, \ldots, v_k \in \R$ to an adaptively-chosen sequence of queries $\phi_1, \ldots, \phi_k: \X^t \rar \R$ satisfying: for all interactive algorithms $A$ and distributions $\cP$ over $\X$, $$\prob{S \sim \cP^n \atop (\phi_{[k]},v_{[k]})\sim\interact{A}{M}(S)}{\forall j \in [k]~ \mbox{s.t.}~ \phi_j \in \F_{\cP,t}:~~~ \left|v_j-\cP^t[\phi_j]\right| \leq 4 \cdot (\phi_j(\cP^t)) + \zeta} \geq 1-\beta.$$
\end{cor}
\begin{proof}
We first observe that by Markov's inequality, $$\pr_{Z\sim \cP^t}\left[|\phi(Z)-\cP^t[\phi]| \geq 4\cdot \MAD(\phi(\cP^t))\right] \leq 1/4.$$
Therefore
\equ{\iqr{\phi(\cP^t)}{\frac 1 4, \frac 3 4} \subseteq \left[\cP^t[\phi]- 4 \cdot \MAD(\phi(\cP^t)), \cP^t[\phi]+ 4 \cdot \MAD(\phi(\cP^t))\right] .\label{eq:iq2mad}}
Hence for all $j\in [k]$ such that $\phi_j \in \F_{\cP,t}$, we have that  $\iqr{\phi_j(\cP^t)}{1/4, 3/4} \subseteq [-5,5]$.
Now we define $T$ to be the interval $[-5,5]$ discretized with step $\zeta$, or $T\doteq \{r \cdot \zeta : r \in \mathbb{Z}\} \bigcap [-5,5]$. To answer a query $\phi_j$ we define $\phi'_j:\X^t \rar T$ as $\phi'_j(z) \doteq \argmin_{v \in T} |v-\phi_j(z)|$ and then use the algorithm from Theorem \ref{thm:adaptive-iqr} to answer the query $\phi'_j$.
The projection of the values of $\phi$ to $T$ simultaneously truncates the range to $[-5,5]$ and discretizes it. The $(1/4,3/4)$-quantile interval of $\phi_j(\cP^t)$ is inside the interval $[-5,5]$ and therefore is not affected by the truncation step. The discretization can affect this interval by at most $\zeta$. Combining this with \eqref{eq:iq2mad} we obtain that if $\phi_j \in \F_{\cP,t}$ then
\equn{\iqr{\phi'_j(\cP^t)}{\frac 1 4, \frac 3 4} \subseteq \left[\cP^t[\phi_j]- 4 \cdot \MAD(\phi_j(\cP^t))-\zeta, \cP^t[\phi_j]+ 4 \cdot \MAD(\phi_j(\cP^t)) +\zeta\right]. }
Therefore the value $v_j$ returned by the algorithm from Theorem \ref{thm:adaptive-iqr} to query $\phi'_j$ satisfies: $$|v_j - \cP^t[\phi_j]| \leq 4 \cdot \MAD(\phi_j(\cP^t)) +\zeta .$$
Now to obtain the claimed bound on the sample complexity we observe that $|T| \leq 10/\zeta$.
\end{proof}
\begin{remark}
Note that mean absolute deviation of $\phi$ is upper-bounded by the standard deviation of $\phi$.  Therefore Corollary \ref{cor:mad-range} also holds with  $\MAD(\phi_j(\cP^t))$ replaced by $\sdv(\phi_j(\cP^t))$ both in the definition of $\F_{t,\cP}$ and the accuracy bound (with the constant factor 4 being replaced by 2 since Chebyshev's inequality can be used instead of Markov's). The obtained statement generalizes Theorem \ref{thm:intro-var} we stated in the introduction.
\end{remark}
The quantile-based guarantees of our other algorithms can be converted to additive error guarantees in an analogous way.

We remark that somewhat sharper (asymptotic) bounds can be obtained by using the approximate median algorithm based on linear queries (Lemma \ref{lem:median2sq}) together with the algorithm for answering linear queries by \citet{SteinkeU16}. Specifically, this algorithm can solve the problem given $n=O\left(t \sqrt{k  \cdot \log (1/\zeta) \cdot \log(1/\beta) \cdot \log(\log(k \log (1/\zeta))/\beta)}\right)$ samples. 

\subsection{Simple Generalization}
We now describe a simple and seemingly weak generalization result that shows the output of a differentially private algorithm cannot ``overfit'' its input dataset. Namely, if an $(\veps,\delta)$ differentially private algorithm outputs a function $\phi:\Z \to \R$ on a dataset $s\in \Z^m$ sampled from $\D^m$, then the value of $\phi$ on any element of the dataset is within the $(\rho,1-\rho)$-quantile interval of $\phi(\D)$ with probability at least $1-(2e^\veps \rho + \delta)$. To obtain meaningful guarantees about the whole dataset from such  generalization result, $\rho$ must be relatively small (much smaller than the desired $1/4$). The good news is that for an estimator that is well-concentrated around its mean, even values within $(\rho,1-\rho)$-quantile interval for small $\rho$ are close to the mean.  Note that, in principle, any estimator can be amplified by sampling and taking a median 
before being used in this analysis and hence we can obtain generalization guarantees from such simple analysis even in the general case (although the algorithm in this case would need to use two median steps).

\begin{thm}\label{thm:gen-simp}
Let $M : \Z^m \to \mathcal{F}_{[0,1]}$ be a $(\varepsilon,\delta)$-differentially private algorithm with $\mathcal{F}_{[0,1]}$ being the set of functions $\phi : \Z \to [0,1]$. Let $\D$ be a distribution on $\Z$. Then for all $i \in [m]$, $$\prob{S \sim \D^m \atop \phi \sim M(S)}{\phi(S_i) \not\in \iqr{\phi(\D)}{\rho,1-\rho}} \leq 2\rho e^\varepsilon + \delta$$
\end{thm}
\begin{proof}
By differential privacy, for all $i \in [m]$, \begin{align*}\prob{S \sim \D^m \atop \phi \sim M(S)}{\phi(S_i) \notin \iqr{\phi(\D)}{\rho,1-\rho}} \leq& e^\varepsilon \prob{(S,Z) \sim \D^m \times \D \atop \phi \sim M(S_{-i},Z)}{\phi(S_i) \notin \iqr{\phi(\D)}{\rho,1-\rho}} + \delta \\ =& e^\varepsilon \prob{(S,Z) \sim \D^m \times \D \atop \phi \sim M(S)}{\phi(Z) \notin \iqr{\phi(\D)}{\rho,1-\rho}} + \delta \\\leq& e^\varepsilon 2\rho + \delta,\end{align*}
where the equalities follow from the fact that the pairs $(S,Z)$ and $((S_{-i},Z),Z_i)$ are identically distributed and the definition of the $(\rho,1-\rho)$-quantile interval.
\end{proof}

Now for $\rho$ and $\delta$ that are sufficiently small, Theorem \ref{thm:gen-simp} ensures that with probability at least $1-\beta$, for all $i\in [m]$, $\phi(s_i) \in \iqr{\phi(\D)}{\rho,1-\rho}$. This means that to get a value in $\iqr{\phi(\D)}{\rho,1-\rho}$, we can use an algorithm that outputs a value that is in between the smallest and the largest values of $\phi$ on the elements of a dataset $s$.  Such value is referred to as an interior point of $\phi(s)$ (and is equivalent to a 1-approximate median).


This argument gives the following theorem.
\begin{thm}
\label{thm:adaptive-simple-general}
For any $\beta \in (0,1)$, $t,k \in \mathbb{N}$, a finite set $T \subset \R$ and $\Z=\X^t$, and with $$n\geq n_0=O\left(t \cdot \sqrt{k} \cdot \log(|T|/\beta)\cdot \log^{1/2}(k\log(|T|)/\beta)\right)$$ there exists an interactive algorithm $M$ that takes as input a dataset $s\in \X^n$ and provides answers $v_1, \ldots, v_k \in T$ to adaptively-chosen queries $\phi_1, \ldots, \phi_k : \X^t \to T$ such that, for all interactive algorithms $A$ and distributions $\cP$ on $\X$, $$\prob{S \sim \cP^n \atop (\phi_{[k]},v_{[k]}) \sim \interact{A}{M}(S)}{\forall j \in [k] ~~~ v_j \in \iqr{\phi_j(\cP^t)}{\rho , 1-\rho}} \geq 1-\beta, $$
where $\rho = \beta \cdot t/(4kn)=\tilde{\Omega}(\beta /(k^{3/2} \cdot \log|T|)).$
\end{thm}
\begin{proof}
We use the algorithm given in Figure \ref{fig:median-strong} but with $1$-approximate median, instead of $1/4$. As in the proof of Theorem \ref{thm:intro-iqr} we let $\Z \doteq \X^t$ and $\D \doteq \cP^t$. Theorem \ref{thm:EM-med} says that if $m \geq 4\ln(2k|T|/\beta)/ \tilde\varepsilon$, then each execution of the median algorithm is $(\tilde\veps,0)$-differentially private
for every input query $\phi_j:\Z \to T_j$. This implies that for all interactive algorithms $A$ and every $s\in \Z^m$, \begin{equation}\prob{(\phi_{[k]},v_{[k]}) \sim \interact{A}{M}(s)}{\forall j \in [k] ~~~ v_j \in \iqr{\phi_j(s)}{0, 1}} \geq 1 - \frac{\beta}{2}.\label{eqn:emp-int-point}\end{equation}
The interactive composition (Theorem \ref{thm:interactive-composition}) implies that $M$ is $(\varepsilon,\delta)$-differentially private for any $\delta \in (0,1)$ and $$\varepsilon = \frac{k}{2} \left( \frac{4\ln(2|T|/\beta)}{m} \right)^2 + \frac{4\ln(2|T|/\beta)}{m} \sqrt{2k\ln(1/\delta)}.$$

Now applying Theorem \ref{thm:gen-simp} and a union bound we obtain that
$$\prob{S \sim \D^m \atop (\phi_{[k]},v_{[k]}) \sim \interact{A}{M}(S)}{\exists j \in [k],\ i\in [m] ~~~ \phi_j(S_i) \not\in \iqr{\phi_j(\D)}{\rho , 1-\rho}} \leq  k m(e^\varepsilon 2\rho + \delta) . $$
Combining this with \eqref{eqn:emp-int-point} we obtain that
$$\prob{S \sim \D^m \atop (\phi_{[k]},v_{[k]}) \sim \interact{A}{M}(S)}{\forall j \in [k] ~~~ v_j \in \iqr{\phi_j(\D)}{\rho , 1-\rho}} \geq 1-\frac{\beta}{2}- k m(e^\varepsilon 2\rho + \delta). $$
Setting $m = 8 \log(2|T|/\beta) \sqrt{2k\ln(1/\delta)}$ ensures that $\varepsilon \leq \ln 2$. Hence for $\delta = \beta/(10km)$ and $\rho = \beta/(10km)$ we obtain that $k m(e^\varepsilon 2\rho + \delta) \leq \beta/2$ thus establishing the claim.
\end{proof}

For example, if each $\phi_j$ is $(1/\sqrt{t})$-subgaussian with the mean $\cP^t[\phi_j]\in [-1,1]$ then for every $\alpha > 0$, setting $t = \tilde{O}(\log(k/\beta)/\alpha^2)$ ensures that $\iqr{\phi_j(\cP^t)}{\rho , 1-\rho} \subseteq [\cP^t[\phi_j]-\alpha,\cP^t[\phi_j]+\alpha]$. This implies that the means can be estimated with accuracy $\alpha$ given $\tilde O\left(\sqrt{k} \cdot \log^2(1/\beta)/\alpha^2\right)$ samples. Note that low-sensitivity queries are $(1/\sqrt{t})$-subgaussian and therefore the sample complexity of our algorithm given by this simple analysis is comparable to the best known for this problem.

As pointed out above, this analysis can also be used to deal with general estimators by adding an additional amplification step. Namely, computing the estimator on several independent subsamples and taking (the exact) median. The resulting algorithm would have sample complexity that is identical to that obtained in Theorem \ref{thm:adaptive-iqr} up to an additional logarithmic factor (which can be removed with careful calibration of parameters).

\remove{

We now show how Theorem \ref{thm:gen-simp} can yield an alternative proof of Theorem \ref{thm:intro-iqr}.
\begin{proof} (of Theorem \ref{thm:intro-iqr}).

The algorithm $M$ promised by Theorem \ref{thm:intro-iqr} is described in Figure \ref{fig:median-strong-simp}.

\begin{figure}[h!]
\begin{framed}
\begin{algorithmic}
\INDSTATE[0]{Input $S \in \X^{n \ell t}$.}
\INDSTATE[0]{Partition $S$ into $S_1^1, \ldots, S_1^\ell, \ldots, S_n^1, \ldots, S_n^\ell \in \X^t$. Let $\Z=\X^t$.}
\INDSTATE[0]{For $j=1,2,\ldots,k$:}
\INDSTATE[1]{Receive a query $\phi_j : \X^t \to T$.}
\INDSTATE[1]{Compute $\phi_j(S_1^1), \ldots, \phi_j(S_1^\ell), \ldots, \phi_j(S_n^1), \ldots, \phi_j(S_n^\ell) \in T$.}
\INDSTATE[1]{Let $a_i = \median(\phi_j(S_i^1), \ldots, \phi_j(S_i^\ell))$ for $i \in [n]$.}
\INDSTATE[2]{(Note that this amplification by median step is not necessary if $\phi_j$ is already concentrated and we only need $v_j \in \iqr{\phi_j(\D)}{\rho,1-\rho}$ for $\rho \ll \beta/n$.)}
\INDSTATE[1]{Run the $(\tilde\varepsilon,0)$-differentially private $1/8$-approximate median algorithm $\tilde M : T^n \to T$ from Theorem \ref{thm:EM-med} with input $(a_1, \ldots, a_n)\in T^n$ to obtain output $v_j \in T$.}
\INDSTATE[1]{Return answer $v_j$.}
\end{algorithmic}
\end{framed}
\vspace{-6mm}
\caption{Algorithm for answering adaptive queries.}\label{fig:median-strong-simp}
\end{figure}

By Lemma \ref{lem:amp-med}, $$\prob{Z_1, \ldots, Z_\ell \sim \D}{\median(\phi(Z_1), \ldots, \phi(Z_\ell)) \in \iqr{\D}{\frac 3 8 ,\frac 5 8}} \geq 1 - 2 \cdot e^{- \ell/8} .$$
Thus, by Theorem \ref{thm:gen-simp} and an argument analogous to that of Theorem \ref{thm:gen-adapt}, $$\prob{}{\forall j \in [k] ~~~ v_j \in \iqr{\phi_j(\D)}{\frac 1 4 ,\frac 3 4}} \geq 1- n(4 \cdot e^{- \ell/8 + \varepsilon} + \delta).$$
where $$\varepsilon = \frac{k}{2} \left( \frac{32\log(k|T|/\beta)}{n} \right)^2 + \frac{32\log(k|T|/\beta)}{n} \sqrt{2k\log(1/\delta)}.$$
It only remains to set parameters: We set $\delta = \beta/2n$, $\varepsilon=\ell/16$, and $\ell=16\log(8n/\beta)$ so that $\prob{}{\forall j \in [k] ~~~ v_j \in \iqr{\phi_j(\D)}{\frac 1 4 ,\frac 3 4}} \geq 1- \beta.$ To achieve this value of $\varepsilon\leq\log(8n/\beta)$, we set $n=64\sqrt{k}\log(k|T|/\beta)/\sqrt{\log(2n/\beta)}$. The overall sample complexity is thus $$n\ell t = \frac{64\sqrt{k}\log(k|T|/\beta)}{\sqrt{\log(2n/\beta)}} \cdot 16\log(8n/\beta) \cdot t \leq 1024 \sqrt{k \log(8n/\beta)} \log(k|T|/\beta).$$
\end{proof}

}

%% file: extensions.tex
\section{Dealing with a Large Number of Queries}\label{sec:extensions}
In this section we briefly cover ways to use our approach when the number of queries that needs to be answered is (relatively) large. Namely, we provide an algorithm for answering verification queries and an algorithm whose complexity scales as $\log k$, rather than $\sqrt k$.

\subsection{Verification Queries}
\label{sec:reusable}

\newcommand{\good}{\mathrm{Y}}
\newcommand{\bad}{\mathrm{N}}
\newcommand{\unsure}{\bot}
Another application of techniques from differential privacy given by \citet{DworkFHPRR14:arxiv} is an algorithm that given a statistical query and a proposed estimate of the expectation of this query, verifies the estimate. This problem requires less data if most proposed answers are correct. Specifically, the number of samples needed by this algorithm is (just) logarithmic in the number of queries $k$ but also scales linearly in $\sqrt{\ell}$, where $\ell$ is the number of queries that fail the verification step. \citet{DworkFHPRR15:arxiv} extended this result to low-sensitivity queries using the results from \citep{BassilyNSSSU16}. In addition, \citet{DworkFHPRR15:arxiv} describe a query verification algorithm that can handle arbitrary queries (not just real-valued) which however has sample complexity with linear dependence on $\ell$. Its analysis is based on a simple description length-based argument.

A natural way to apply such algorithms is the reusable holdout technique of \citet{DworkFHPRR15:science}. In this technique the dataset is split into two disjoint parts: the ``training" set $s_t$ and the holdout set $s_h$. The analyst then uses the training set to answer queries and perform other arbitrary analyses. The holdout set is used solely to check whether the answers that were obtained on the training set generalize. Another application proposed by \citet{DworkFHPRR14:arxiv} is an algorithm referred to as {\tt EffectiveRounds}. This algorithm splits the dataset into several disjoint subsets and at each time uses only one of the subsets to answer queries. An algorithm for verifying answers to queries is  used to switch to a new subset of samples whenever a query fails the verification step (and uses its own subset of samples).

Here we demonstrate, an algorithm for verifying answers to queries about  general estimators.  Formally, our algorithm satisfies the following guarantees.
\begin{thm}\label{thm:sv}
Fix $\rho>\alpha>0$, $\beta>0$, $\ell,t,k \in \mathbb{N}$, and $n \geq n_0= O(t \sqrt{\ell \log(1/\alpha\beta)} \log(k/\beta)\rho/\alpha^2)$.
There exists an interactive algorithm $M$ that takes as input $s \in \X^n$ and provides answers $a_1, \ldots, a_k \in \{\good,\bad,\unsure\}$ to adaptively-chosen queries $(\phi_1,v_1), \ldots, (\phi_k,v_k)$ (where $\phi_j : \X^t \to \R$ and $v_j \in \R$ for all $j \in [k]$) satisfying the following: for all interactive algorithms $A$ and distributions $\cP$ over $\X$, $$\prob{S \sim \cP^n \atop (\phi_{[k]},v_{[k]},a_{[k]}) \sim \interact{A}{M}(S)}{\forall j \in [k] ~~~ \begin{array}{rl} v_j \in \iqr{\phi_j(\cP^t)}{\rho,1-\rho} &\implies a_j \in \{\good,\unsure\} \\ v_j \notin \iqr{\phi_j(\cP^t)}{\rho-\alpha,1-\rho+\alpha} &\implies a_j \in \{\bad,\unsure\} \\ |\{ j' \in [j-1] : a_{j'} = \bad \}|= \ell &\iff a_j=\unsure \end{array}} \geq 1-\beta. $$
\end{thm}
First, we explain the accuracy promise of this algorithm. Each query is specified by a function $\phi_j : \X^t \to \R$ as well as a ``guess'' $v_j \in \R$. The guess is ``good'' if $v_j \in \iqr{\phi_j(\cP^t)}{\rho,1-\rho}$ and ``bad'' if $v_j \notin \iqr{\phi_j(\cP^t)}{\rho-\alpha,1-\rho+\alpha}$. Essentially, the guarantee is that, if the guess is good, the algorithm answers $\good$ and, if the guess is bad, the algorithm answers $\bad$. However, there are two caveats to this guarantee --- (i) if the guess is neither good nor bad, then the algorithm may output either $\good$ or $\bad$ and, (ii) once the algorithm has given the answer $\bad$ to $\ell$ queries, its failure budget is ``exhausted'' and it only outputs $\unsure$.

Note that this algorithm handles only the verification and does not provide correct answers to queries that failed the verification step. To obtain correct responses one can run an instance of the query-answering algorithm from Theorem \ref{thm:intro-iqr} in parallel with the verification algorithm. The query-answering algorithm would be used only $\ell$ times and hence the dataset size it would require would be independent of $k$ and scale linearly in $\sqrt{\ell}$. (The two algorithms can either be run on disjoint subsets of data or on the same dataset since differential privacy composes.)

Our proof is a simple reduction of the verification step to verification of answers to statistical queries (relative to $\cP^t$). Hence we could directly apply results from \citep{DworkFHPRR15:arxiv} to analyze our algorithm. However, as we mentioned above, for small values of $\rho$ the existing algorithm has suboptimal dependence of sample complexity on $\rho$ and $\alpha$. Specifically, the dependence is $1/\alpha^2$ instead of $\rho/\alpha^2$ which, in particular, is quadratically worse for the typical setting of $\alpha = \Theta(\rho)$.
Note that the dataset size grows with $\ell$ -- the number of times that ``No" is returned to a verification query. Therefore choosing a small $\rho$ is useful for ensuring that ``No" is returned only when overfitting is substantial enough to require correction.

To improve this dependence we use our sharper generalization result (Theorem \ref{thm:gen-adapt}). But first we need to state the stability properties of the sparse vector technique that is the basis of this algorithm \citep{DworkNRRV09} (see \cite[\S 3.6]{DworkRoth:14} for a detailed treatment).

\begin{thm}[\citep{DworkNRRV09}]\label{thm:sv-emp}
For all $\alpha,\beta,\varepsilon,\delta>0$, $m,k,\ell \in \mathbb{N}$ with
$$m \geq m_0 = O(\sqrt{\ell\log(1/\delta)}\log (k/\beta)/\varepsilon\alpha),$$
there exists an interactive $(\varepsilon,\delta)$-differentially private algorithm $\tilde M$ that takes as input $s \in \Z^m$ and provides answers $b_1, \ldots, b_k \in \{\good,\bad,\unsure\}$ to adaptively chosen queries $(\psi_1,u_1), \ldots, (\psi_k,u_k)$ (where $\psi_j : \Z \to \R$ and $u_j \in \R$ for all $j \in [k]$) with the following accuracy guarantee. For all interactive algorithms $A$ and all $s \in \Z^m$, $$\prob{(\psi_{[k]},u_{[k]},b_{[k]}) \sim \interact{A}{\tilde M}(s)}{\forall j \in [k] ~~~ \begin{array}{rl} s[\psi_j] > u_j &\implies b_j \in \{\good,\unsure\} \\ s[\psi_j] \leq u_j - \alpha &\implies b_j \in \{\bad,\unsure\} \\ |\{ j' \in [j-1] : b_{j'} = \bad \}|= \ell &\iff b_j=\unsure \end{array}} \geq 1-\beta. $$
\end{thm}

\remove{
First we use Theorem \ref{thm:gen} to convert the algorithm of Theorem \ref{thm:sv-emp} which has an empirical accuracy guarantee into a verification algorithm for statistical queries with an accuracy guarantee relative to the distribution. This step is very similar to Theorem \ref{thm:gen-adapt}.

\begin{thm}\label{thm:sv-sq}
For all $\beta,\rho,\alpha \in (0,1)$ with $\alpha<\rho$, $m,k,\ell \in \mathbb{N}$ with
$$m \geq m_0 = O(\sqrt{\ell\log(1/\alpha\beta)}\log (k/\beta)\rho/\alpha^2),$$
there exists an interactive $(\varepsilon,\delta)$-differentially private algorithm $\tilde M$ that takes as input $s \in \Z^m$ and provides answers $b_1, \ldots, b_k \in \{\good,\bad,\unsure\}$ to adaptively chosen queries $\psi_1, \ldots, \psi_k$ (where $\psi_j : \Z \to [0,1]$ for all $j \in [k]$) with the following accuracy guarantee. For all interactive algorithms $A$ and all distributions $\D$ on $\Z$, $$\prob{S \sim \D^m \atop (\psi,b) \sim \interact{A}{\tilde M}(S)}{\forall j \in [k] ~~~ \begin{array}{rl} \D[\psi_j] > \rho &\implies b_j \in \{\good,\unsure\} \\ \D[\psi_j] \leq \rho-\alpha &\implies b_j \in \{\bad,\unsure\} \\ |\{ j' \in [j-1] : b_{j'} = \bad \}|= \ell &\iff b_j=\unsure \end{array}} \geq 1-\beta. $$
\end{thm}
\begin{proof}
Let $\tilde M$ be the algorithm promised by Theorem \ref{thm:sv-emp} instantiated with ???? and with $u_j=???$ for all $j$.

Let $\cQ$ be the set of functions $\phi : \Z \to \R$ so that $\interact{A}{\tilde M} : \Z^m \to \cQ^k \times \{\good,\bad,\unsure\}^k$. We define $f : {\cQ}^k \times \{\good,\bad,\unsure\}^k \to \mathcal{F}_{[0,1]}$ as follows, where $\mathcal{F}_{[0,1]}$ is the set of functions $\psi : \Z \to [0,1]$. Given $(\phi,v) \in \cQ^k \times \R^k$, define $\psi_1,\psi_{-1},\psi_2,\psi_{-2}, \ldots, \psi_k, \psi_{-k} : \Z \to \{0,1\}$ by $$\psi_j(x) = \phi_j(x) ~~~\text{ and }~~~ \psi_{-j}(x) = 1-\phi_j(x)$$ and let $$f(\phi,v) = \argmin_{\psi \in \{\psi_1,\psi_{-1},\psi_2,\psi_{-2}, \ldots, \psi_k, \psi_{-k}\}} \D[\psi].$$
By the postprocessing property of differential privacy (Theorem \ref{thm:post}), $f(\interact{A}{\tilde M}(s))$ is a $(\varepsilon,\delta)$-differentially private algorithm (relative to its input $s \in \Z^m$). Moreover, by our assumption \eqref{eqn:empacc}, $$\forall A ~~ \forall s \in \Z^m ~~~~~ \prob{\psi \sim f\left(\interact{A}{M}(s)\right)}{s[\psi] \geq \frac 3 8} \geq 1- \beta.$$
However, by Theorem \ref{thm:gen}, $$\prob{S \sim \D^m \atop \psi \sim f(\interact{A}{M}(x))}{S[\psi]-\D[\psi] > \frac{1}{8}} \leq \prob{S \sim \D^m \atop \psi \sim f(\interact{A}{M}(x))}{S[\psi]-\D[\psi] > \frac12 \left(e^{2\varepsilon}-1\right) + \frac{4}{\varepsilon m} \ln\left(1+\frac{1}{\beta}\right) + 8\frac{\delta}{\beta} } \leq \beta.$$ Thus, by a union bound and the construction of $f$, $$\prob{S \sim \D^m \atop (\phi_{[k]},v_{[k]}) \sim \interact{A}{M}(S)}{\forall j \in [k] ~~~ v_j \in \iqr{\phi_j(\D)}{\frac 1 4 , \frac 3 4}} = \prob{S \sim \D^m \atop \psi \sim f(\interact{A}{M}(S))}{\D[\psi] \geq \frac 1 4} \geq 1- 2 \beta.$$
\end{proof}
}

\begin{proof}[of Theorem \ref{thm:sv}.]
As before, we let $\Z \doteq \X^t$, $\D \doteq \cP^t$ and view the input dataset as an element of $\Z^m$ sampled from $\D^m$.

\vnote{I think it the proof would be a bit cleaner and easier to follow if we prove it for queries that verify only one side of quantile and then just say that the other side is verified symmetrically (or, equivalently, by verifying it for query $-\phi$).}\tnote{I've tried adding some more explanatory text.} \vnote{I think the approach is pretty clear anyway, I was mostly commenting about removing some of the repetition.}
We use the sparse vector algorithm of Theorem \ref{thm:sv-emp}. We use Theorem \ref{thm:gen} to convert the empirical guarantee into a guarantee relative to the data distribution, as in Theorem \ref{thm:gen-adapt}.
However, we must convert every verification query into two statistical queries to check the two ``ends'' of the quantile interval.  That is, given a verification query $(\phi_j,v_j)$ for which we want to know whether $v_j \in \iqr{\phi_j(\D)}{\rho,1-\rho}$ or $v_j \notin \iqr{\phi_j(\D)}{\rho-\alpha,1-\rho+\alpha}$, we ask two statistical queries $\psi_{2j-1}$ and $\psi_{2j}$ for which we want to know if $\D[\psi] > \rho$ or $\D[\psi] < \rho-\alpha$ (with $\psi \in \{\psi_{2j-1},\psi_{2j}\}$). Formally, we have the following reduction.

We convert the sparse vector algorithm $\tilde M$ of Theorem \ref{thm:sv-emp} into an algorithm $M$ of the desired form of Theorem  \ref{thm:sv}: Each query $(\phi_j,v_j)$ to $M$ is converted into two queries $(\psi_{2j-1},u_{2j-1})$ and $(\psi_{2j},u_{2j})$ to $\tilde M$, where $$\psi_{2j-1}(z) \doteq \ind(\phi_j(x) \leq v_j) ,~~~~ \psi_{2j}(z) \doteq \ind(\phi_j(x) \geq v_j) , ~~~~ u_{2j-1} \doteq u_{2j} \doteq \rho-\alpha/3.$$
These queries have the following key property: ($\dagger$) If $v_j \in \iqr{\phi_j(\D)}{\rho,1-\rho}$, then $\D[\psi_{2j-1}]>\rho$ and $\D[\psi_{2j}]>\rho$. If $v_j \notin \iqr{\phi_j(\cP^t)}{\rho-\alpha,1-\rho+\alpha}$, then either $\D[\psi_{2j-1}]\leq\rho-\alpha$ or $\D[\psi_{2j}]\leq\rho-\alpha$ (but not both).

Let $b_{2j-1}$ and $b_{2j}$ be the answers produced by $\tilde M$ to $(\psi_{2j-1},u_{2j-1})$ and $(\psi_{2j},u_{2j})$ respectively. If $b_{2j-1}=b_{2j}=\good$, then $M$ returns $a_j = \good$. If $b_{2j-1}=\unsure$ or $b_{2j}=\unsure$ (or both), then $M$ returns $a_j = \unsure$. Otherwise $M$ returns $a_j=\bad$.

Note that $\tilde M$ must answer twice as many queries as $M$; thus $\tilde M$ must be instantiated with the value $k$ being twice as large as for $M$. We also instantiate $\tilde M$ with $\alpha$ reduced by a factor of 3 and $\beta$ reduced by a factor of $2$. The values $\varepsilon$ and $\delta$ used by $\tilde M$ will be determined later in this proof.

In particular, $\tilde M$ is instantiated to achieve the following accuracy guarantee. For all interactive algorithms $A$ and all $s \in \Z^m$, $$\prob{(\psi_{[k]},u_{[k]},b_{[k]}) \sim \interact{A}{\tilde M}(s)}{\forall j \in [k] ~~~ \begin{array}{rl} s[\psi_j] > u_j &\implies b_j \in \{\good,\unsure\} \\ s[\psi_j] \leq u_j - \alpha/3 &\implies b_j \in \{\bad,\unsure\} \\ |\{ j' \in [j-1] : b_{j'} = \bad \}|= \ell &\iff b_j=\unsure \end{array}} \geq 1-\frac{\beta}{2}. $$

Now we prove that $M$ satisfies the promised accuracy requirement relative to the distribution $\D$ rather than relative to the empirical values. Given the key property ($\dagger$) above, it suffices to show that, with probability at least $1-\beta/2$ over a random choice of $S \sim \D^m$, for all $j \in [2k]$ and $(\psi_{[k]},u_{[k]},b_{[k]}) \sim \interact{A}{\tilde M}(S)$, we have \begin{equation}\D[\psi_j]>\rho \implies S[\psi_j]>u_j=\rho-\alpha/3 ~~~\text{ and }~~~ \D[\psi_j]\leq \rho-\alpha \implies S[\psi_j] \leq u_j-\alpha/3 = \rho-2\alpha/3.\label{eqn:suff}\end{equation}
Furthermore, to prove \eqref{eqn:suff}, it suffices to have $$\D[\psi_j] \cdot \frac{\rho-\alpha/4}{\rho} - \frac{\alpha}{12} \leq S[\psi_j]  \leq \D[\psi_j] \cdot \frac{\rho-3\alpha/4}{\rho-\alpha} + \frac{\alpha}{12},$$ which is, in turn, implied by \begin{equation}\left| S[\psi_j] - \D[\psi_j] \right| \leq \frac{\alpha}{4\rho} \cdot \D[\psi_j]  + \frac{\alpha}{12}.\label{eqn:multadd}\end{equation}

We will now use Theorem \ref{thm:gen} to prove that \eqref{eqn:multadd} holds simultaneously for all $j \in [2k]$ with probability at least $1-\beta/2$, as required to complete the proof.

First we define $f : \mathcal{F}_{\{0,1\}}^{2k} \times \mathbb{R}^{2k} \times \{\good,\bad,\unsure\}^{2k} \to \mathcal{F}_{\{0,1\}}^{4k}$ by $f(\psi,u,b) = (\psi_1,1-\psi_1,\psi_2,1-\psi_2,\ldots,\psi_{2k},1-\psi_{2k})$. By postprocessing (Theorem \ref{thm:post}), $f(\interact{A}{M}(s))$ is an $(\varepsilon,\delta)$-differentially private function of $s \in \Z^m$ for all interactive algorithms $A$. The output of $f(\interact{A}{M}(s))$ is $4k$ functions mapping $\Z$ to $\{0,1\}$. Set $\varepsilon = \frac12\ln\left(1+\frac{\alpha}{8\rho}\right)$ and $\delta=\alpha\beta/(16 \cdot 2 \cdot 12)$. Suppose $m \geq \frac{8 \cdot 12}{\varepsilon\alpha} \ln(16k/\beta)$. By Theorem \ref{thm:gen}, \begin{equation}\prob{S \sim \D^m \atop \hat\psi \sim f(\interact{A}{M}(S))}{\forall j \in [2k] ~~~~ \begin{array}{rl} S[\psi_j]-\D[\psi_j] &\leq \frac{\alpha}{12} + \frac{\alpha}{8\rho} \cdot \MAD(\psi_j(\D)) \\S[1-\psi_j]-\D[1-\psi_j] &\leq \frac{\alpha}{12} + \frac{\alpha}{8\rho} \cdot \MAD(\psi_j(\D)) \end{array}} \geq 1-\frac{\beta}{2}.\label{eqn:gne}\end{equation}
Note that $\max\{S[\psi_j]-\D[\psi_j],S[1-\psi_j]-\D[1-\psi_j]\} = |S[\psi_j]-\D[\psi_j]|$ and $\MAD$ of $\psi_j$ is equal to $\MAD$ of $1-\psi_j$. Since $\MAD(\psi_j(\D))  \leq 2 \cdot\D[\psi_j]$, the generalization bound \eqref{eqn:gne} implies the desired bound \eqref{eqn:multadd} holds simultaneously for all $j \in [2k]$ with probability at least $1-\beta/2$, as required.

It only remains to work out the parameters. We have $\varepsilon \geq \alpha/18\rho$. Theorem \ref{thm:gen} requires $m \geq m_1$ where $m_1 = \frac{8 \cdot 12}{\varepsilon\alpha} \ln(16k/\beta) \leq 8 \cdot 12 \cdot 18 \frac{\rho}{\alpha^2} \ln(16k/\beta)$, while Theorem \ref{thm:sv-emp} requires $$m \geq m_0 = O(\sqrt{\ell\log(1/\delta)}\log (k/\beta)/\varepsilon\alpha) = O(\sqrt{\ell\log(1/\alpha\beta)}\log(k/\beta) \rho / \alpha^2).$$ Thus the final sample complexity is $\max\{m_0,m_1\}$, as required.
\end{proof}

\subsection{Private Multiplicative Weights}

Now we present a result in which the dependence of the required dataset size on the number of queries $k$ is logarithmic (at the expense of some additional terms and computational efficiency).\footnote{It is known that in this general setting, the dependence on the data universe size and the loss of computational efficiency are unavoidable \citep{HardtU14,SteinkeU15}.}

Our result follows from a direct combination of an algorithm for answering statistical queries from \citep{DworkFHPRR14:arxiv} and the reduction from the approximate median problem to the problem of answering statistical queries relative to distribution $\D = \cP^t$ (given in Lemma \ref{lem:median2sq}).

Specifically, we rely on the following result from \citep{DworkFHPRR14:arxiv,BassilyNSSSU16} that is based on the private multiplicative weights algorithm of \citet{HardtR10} (see also \cite[\S 4.2]{DworkRoth:14} for further exposition).

\begin{thm}[{\cite[Corollary 6.3]{BassilyNSSSU16}}]
\label{thm:adaptive-sq-mwu}
For all $\alpha,\beta \in (0,1)$ and $m,k \in \mathbb{N}$ with $$m\geq m_0 =O(\sqrt{\log|\Z|} \cdot \log k \cdot \log^{3/2}(1/(\alpha\beta))/\alpha^3),$$ there exists an interactive algorithm $M$ that takes as input a dataset $s \in \Z^m$ and provides answers $v_1, \ldots, v_k \in [-1,1]$ to adaptively-chosen queries $\psi_1, \ldots, \psi_k : \Z \to [-1,1]$ such that, for all interactive algorithms $A$ and distributions $\D$ over $\Z$, $$\prob{S \sim \D^m \atop (\psi_{[k]},v_{[k]}) \sim \interact{A}{M}(S)}{\forall j \in [k] ~~~ |v_j - \D[\psi_j(\D)] \leq \alpha} \geq 1-\beta. $$
\end{thm}

Now, by Lemma \ref{lem:median2sq}, for $\Z = \X^t$ and $\D = \cP^t$ and any query $\phi:\Z \rar T$, responses to $2\lceil\log_2|T|\rceil$ statistical queries relative to $\D$ with accuracy $1/8$ can be used to find a value $v\in \iqr{\phi(\D)}{1/4,3/4}$. By plugging this reduction into Theorem \ref{thm:adaptive-sq-mwu} we get the following result.

\begin{thm}
\label{thm:adaptive-general-mwu}
For any $\beta \in (0,1)$, $t,k \in \mathbb{N}$, a finite set $T \subset \R$ and $\Z=\X^t$ and with $$n\geq n_0=O\left(t^{3/2}\cdot \sqrt{\log|\X|} \cdot \log(k\log|T|)\cdot \log^{3/2}(1/\beta)\right)$$ there exists an interactive algorithm $M$ that takes as input a dataset $s\in\cP^n$ and provides answers $v_1, \ldots, v_k \in T$ to adaptively-chosen queries $\phi_1, \ldots, \phi_k : \X^t \to T$ such that, for all interactive algorithms $A$ and distributions $\cP$ over $\X$, $$\prob{S \sim \cP^n \atop (\phi_{[k]},v_{[k]}) \sim \interact{A}{M}(S)}{\forall j \in [k] ~~~ v_j \in \iqr{\phi_j(\cP^t)}{\frac 1 4 , \frac 3 4}} \geq 1-\beta. $$
\end{thm}

For example we can use this algorithm to obtain a new algorithm for answering a large number of  low-sensitivity queries (that is queries $\phi:\X^t \rar [-1,1]$ such that $\Delta(\phi) = 1/t$). To answer queries with accuracy $\alpha$ we can use $t =16/\alpha^2$ and set $T$ that is the interval $[-1,1]$ discretized with step $\alpha/2$. Thus the number of samples that our algorithm uses is $n = O\left(\sqrt{\log|\X|} \cdot \log(k/\alpha)\cdot \log^{3/2}(1/\beta)/\alpha^{3}\right)$.  For comparison, the best previously known algorithm for this problem uses $$n = O\left(\log|\X| \cdot \log(k/\alpha)\cdot \log^{3/2}(1/\beta)/\alpha^{4}\right)$$ \citep{BassilyNSSSU16} (a different bound is stated there in error). Although, as pointed out in the introduction, the setting in which each query is applied to the entire dataset is more general than ours.

\remove{
\begin{proof}
The idea here is to use Private Multiplicative Weights (Theorem \ref{thm:pmw}) along with our generalization result for differential privacy (Theorem \ref{thm:gen-adapt}) to implement an ``oracle'' for $\D$ using the sample $S$. The oracle approximately answers adaptively-chosen statistical queries $f : \Z \to \{0,1\}$ of the form $$f(z) = \mathbb{I}[\phi(z) \leq v] ~~~\text{ or }~~~ f(z) = \mathbb{I}[\phi(z) \geq v] $$ for $v \in \R$. Checking the parameters, we see that with the sample size we have afforded ourselves it is possible to use private multiplicative weights and strong generalization to implement an oracle that answers $O(k\log|T|)$ such queries and, with probability at least $1-\beta$, for every query $f_j$ and corresponding answer $b_j$ we have $|f_j(\D)-b_j| \leq 1/5 $. Once we have this oracle, we use binary search to answer queries: For every query $\phi_j : \Z \to T$ to $M$, we make $O(\log|T|)$ queries to the oracle in order to find an answer $v_j \in T$ such that $\prob{Z \sim \D}{\phi(Z) \leq v_j} > 1/4$ and $\prob{Z \sim \D}{\phi(Z) \geq v_j} > 1/4$.
\end{proof}
} 

%% file: mad-appendix.tex
\section{Proof of Theorem \ref{thm:gen}}
\vnote{Need to update mad related notation}
Recall that, for a distribution $\D$ on $\R$, we define its mean absolute deviation by $$\MAD(\D) \doteq \E_{X \sim \D}[|X-\ex{Y \sim \D}{Y}|].$$
We first give a bound relating differential privacy to expectations:
\label{app:mad-proof}
\begin{lem} \label{lem:DPEX}
Fix $\mu,\varepsilon,\delta, \Delta \in \mathbb{R}$. Let $X$ and $Y$ be random variables supported on $[\mu-\Delta,\mu+\Delta]$. Suppose that $X$ and $Y$ are $(\veps,\delta)$-indistinguishable, that is $$e^{-\varepsilon} \left( \prob{}{X \in E} - \delta \right) \leq \prob{}{Y \in E} \leq e^\varepsilon \prob{}{X \in E} + \delta$$ for all $E \subseteq \mathbb{R}$. Then $$\left|\ex{}{X} - \ex{}{Y}\right| \leq (e^\varepsilon-1) \ex{}{|X-\mu|} + 2\delta \Delta.$$
\end{lem}
Thus, if $M : \Z^m \to [0,1]$ satisfies $(\varepsilon,\delta)$-differential privacy, then for any datasets $s,s' \in \Z^m$ differing on a single entry, we have $$|\ex{}{M(s)}-\ex{}{M(s')}| \leq (e^\varepsilon-1)\inf_{\mu \in [0,1]} \ex{}{|M(s)-\mu|} + 2\delta \leq (e^\varepsilon-1) \cdot \MAD(M(s)) + 2\delta.$$
In contrast, \citet{BassilyNSSSU16} use a bound corresponding to $|\ex{}{M(S)}-\ex{}{M(S')}| \leq e^\varepsilon-1 + \delta$.
\begin{proof}
We use three facts: (i) $x = \max\{x,0\}-\max\{-x,0\}$ and $|x| = \max\{x,0\}+\max\{-x,0\}$ for all $x \in \mathbb{R}$, (ii) if $X \geq 0$, then $\ex{}{X} = \int_0^\infty \prob{}{X\geq t} \mathrm{d}t$, and (iii) $e^\varepsilon-1 \geq 1- e^{-\varepsilon}$.
\begin{align*}
\ex{}{Y-\mu} =& \ex{}{ \max\{Y-\mu,0\} - \max\{\mu-Y,0\} }\\
=& \int_0^\infty \prob{}{Y-\mu \geq t} - \prob{}{\mu-Y \geq t} \mathrm{d}t\\
=& \int_0^\Delta \prob{}{Y-\mu \geq t} - \prob{}{\mu-Y \geq t} \mathrm{d}t\\
\leq& \int_0^\Delta (e^\varepsilon \prob{}{X-\mu \geq t} + \delta) - e^{-\varepsilon} (\prob{}{\mu-X \geq t} - \delta) \mathrm{d}t\\
=& \int_0^\Delta \prob{}{X-\mu \geq t} - \prob{}{\mu-X \geq t} \mathrm{d}t\\
&+ \int_0^\Delta (e^\varepsilon-1) \prob{}{X-\mu \geq t} + \delta + (1-e^{-\varepsilon}) \prob{}{\mu-X \geq t} + e^{-\varepsilon}\delta \mathrm{d}t\\
=& \ex{}{X-\mu} + (e^\varepsilon-1) \ex{}{\max\{X-\mu,0\}} + (1-e^{-\varepsilon}) \ex{}{\max\{\mu-X,0\}} + (1+e^{-\varepsilon})\delta\Delta \\
\leq& \ex{}{X-\mu} + (e^\varepsilon-1) \ex{}{\max\{X-\mu,0\}+\max\{\mu-X,0\}} + 2\delta\Delta \\
=& \ex{}{X-\mu} + (e^\varepsilon-1) \ex{}{|X-\mu|} + 2\delta\Delta .
\end{align*}
Thus $\ex{}{Y} - \ex{}{X} \leq (e^\varepsilon-1) \ex{}{|X-\mu|} + 2\delta\Delta$. To obtain the other half of the result, replace $X$, $Y$, and $\mu$ with their negations in the above.
\end{proof}

Intuitively, the following lemma says the following. Suppose a differentially private algorithm is given $\ell$ independent samples $S^1, \ldots, S^\ell \sim \D^m$. The algorithm picks one of the $\ell$ samples and produces a statistical query. The algorithm's ``goal'' is to ``overfit'' --- that is, to produce a query whose empirical value on the chosen sample differs from the expected value on the population. The lemma says that this cannot happen in expectation.
The reason for the $\ell$-fold repetition is probability amplification: The lemma says the mechanism cannot overfit in expectation, given $\ell$ chances to do so. Consequently, if the mechanism is given only one chance to overfit (i.e.~one sample $S \sim \D^m$), then with high probability it cannot. The $\ell$ repetitions mean that, if the mechanism can overfit with probability $1/\ell$ per sample, then it can overfit with constant probability given $\ell$ samples.

\begin{lem} \label{lem:MultiEx}
Let $M : (\Z^m)^\ell \to [\ell]\times\mathcal{F}_{[-1,1]}$ be a $(\varepsilon,\delta)$-differentially private algorithm with $\mathcal{F}_{[-1,1]}$ being the set of functions $\phi : \mathcal{Z} \to [-1,1]$. Let $\D$ be a distribution on $\Z$. Then $$\ex{S^1, \ldots, S^\ell \sim \D^m \atop (k,\phi) \sim M(S)}{S^k[\phi]-\D[\phi]} \leq (e^\varepsilon-1)  {\ex{(S,Z) \sim (\D^m)^\ell \times \D \atop (k,\phi) \sim M(S)}{|\phi(Z)|}} + 2\delta m.$$
\end{lem}
\begin{proof}
For $j \in [\ell]$, define $f_j : [\ell] \times \mathcal{F}_{[-1,1]} \times \Z \to [-1,1]$ by $f_j(k,\phi,x)=\phi(x) \cdot \ind(k=j)$.

We use two facts: \begin{itemize}\item[(i)] For $(S,Z) \sim (\D^m)^\ell \times \D$, the pair $(S,S_i^j)$ has the same distribution as the pair $((S_{-i}^{-j},Z),Z)$, where $(S_{-i}^{-j},Z)$ denotes $S$ with the $(i,j)^\text{th}$ entry $S_i^j$ replaced by $Z$. \item[(ii)] By differential privacy of $M$, for all fixed $(s,z) \in (\Z^m)^\ell \times \Z$, the distribution of $(M(s_{-i}^{-j},z),z)$ is $(\varepsilon,\delta)$-indistinguishable from $(M(s),z)$. Hence, for a random pair $(S,Z) \sim (\D^m)^\ell \times \D$, $(M((S_{-i}^{-j},Z),Z)$ is also $(\varepsilon,\delta)$-indistinguishable from $(M(S),Z)$. 
\end{itemize}
Now
\begin{align*}
\ex{S \sim (\D^m)^\ell \atop (k,\phi) \sim M(S)}{S^k[\phi]-\D[\phi]}
=& \sum_{j \in [\ell]} \frac{1}{m} \sum_{i \in [m]} \ex{S \sim (\D^m)^\ell \atop (k,\phi) \sim M(S)}{(\phi(S_i^j) - \D[\phi]) \cdot \ind(k=j)} \\
=& \sum_{j \in [\ell]} \frac{1}{m} \sum_{i \in [m]} \ex{S \sim (\D^m)^\ell \atop Z \sim \D}{f_j(M(S),S_i^j) - f_j(M(S),Z)} \\
\text{(Fact (i))}~~~~~~=& \sum_{j \in [\ell]} \frac{1}{m} \sum_{i \in [m]} \ex{S \sim (\D^m)^\ell \atop Z \sim \D}{f_j(M(S_{-i}^{-j},Z),Z) - f_j(M(S),Z)} \\
\text{(Fact (ii) and Lemma \ref{lem:DPEX})}~~~~~~\leq& \sum_{j \in [\ell]} \frac{1}{m} \sum_{i \in [m]} (e^\varepsilon-1)  \ex{S \sim (\D^m)^\ell \atop Z \sim \D}{| f_j(M(S),Z) |} + 2\delta\\
=&  (e^\varepsilon-1)  {\ex{(S,Z) \sim (\D^m)^\ell \times \D \atop (k,\phi) \sim M(S)}{|\phi(Z)|}} + 2\delta \ell.
\end{align*}
\end{proof}

Now we can state and prove the ``transfer theorem'' which relates differential privacy to generalization. This result improves the previous bound \citep{BassilyNSSSU16} by having a bound in terms of the mean absolute deviation, rather than the sensitivity, (i.e.~the previous bound can be obtained (up to constants) by replacing $\MAD(\cdot)$ in the expression by $1$). Also, the previous bound is only stated for $k=1$.
\begin{thm}[Theorem~\ref{thm:gen}]
Fix $\beta,\varepsilon,\delta \in (0,1)$ and $m,k \in \mathbb{N}$.
Let $M : \Z^m \to \mathcal{F}_{[0,1]}^k$ be a $(\varepsilon,\delta)$-differentially private algorithm with $\mathcal{F}_{[0,1]}$ being the set of functions $\phi : \Z \to [0,1]$. Let $\D$ be a distribution on $\Z$. Then $$\prob{S \sim \D^m \atop \phi \sim M(S)}{\exists j \in [k] ~~~~ S[\phi_j]-\D[\phi_j] > \alpha + \gamma \cdot \MAD\left(\phi_j(\D)\right)} \leq \beta$$ for $\alpha = \frac{4}{\varepsilon m} \ln\left(2\frac{k}{\beta}\right) + 8\frac{\delta}{\beta}$, $\gamma=e^{2\varepsilon}-1$ and $\beta \in (0,1)$ arbitrary.
\end{thm}
We note that the statement of Theorem \ref{thm:gen} has been rearranged to make $\alpha$ and $\gamma$ independent parameters and $\varepsilon$ and $\delta$ dependent parameters, rather than the reverse.
\begin{proof}
Let $\ell=\lceil 1/\beta \rceil$. Define $M' : (\Z^m)^\ell \to [\ell] \times \mathcal{F}_{[-1,1]}$ as follows. On input $S$ it runs $\ell$ copies of $M$ on $S^1, \ldots, S^\ell$ and obtains outputs $\phi^1, \ldots, \phi^\ell \in \mathcal{F}_{[0,1]}^k$. Also define $\phi^0_0$ to be the constant 0 function and $S^0$ to be an arbitrary fixed element of $\Z^m$. Now $M'$ randomly samples $(I,J) \in ([\ell] \times [k]) \cup \{(0,0)\}$ with $$\prob{}{I=i \wedge J=j} \propto \exp\left(\frac{\varepsilon m}{2} \left(S^i[\phi^{i}_j]-\D[\phi^{i}_j]-\gamma\cdot  \MAD\left(\phi^{i}_j(\D)\right)\right)\right).$$ Finally, $M'$ returns $(\max\{I,1\},\phi^*)$ where $\phi^*(x) = \phi^{I}_{J}(x) - \D[\phi^{I}_{J}]$.

Firstly, $M'$ satisfies $(2\varepsilon,\delta)$-differential privacy: The choice of $I$ and $J$ is $(\varepsilon,0)$-differentially private, as it is an instantiation of the exponential mechanism \citep{McSherryTalwar:07}. Simple composition property of differential privacy \cite[Theorem 3.16]{DworkRoth:14} then implies this privacy bound.

Moreover, by the properties of the exponential mechanism \cite[Lemma 7.1]{BassilyNSSSU16} \begin{align} &\ex{I,J}{S^{I}[\phi^{I}_{J}]-\D[\phi^{I}_{J}]-\gamma\cdot \MAD\left(\phi^{I}_{J}(\D)\right)}\nonumber\\ &~~~~~\geq \max_{(i,j) \in [\ell] \times [k] \cup \{(0,0)\}} S^i[\phi^i_j] - \D[\phi^i_j] -\gamma\cdot  \MAD\left(\phi^{i}_j(\D)\right)- \frac{2}{\varepsilon m} \ln (\ell \cdot k + 1). \label{eqn:EM}\end{align}
On the other hand, by Lemma \ref{lem:MultiEx}, \begin{equation} \ex{S \sim (\D^m)^\ell \atop (I,\phi^*) \sim M'(S)}{S^{I}[\phi^*] - \D[\phi^*]} \leq (e^{2\varepsilon}-1) \ex{(S,Z) \sim (\D^m)^\ell \times \D \atop (I,\phi^*) \sim M'(S)}{|\phi^*(Z)|} + 2\delta m.\label{eqn:MultiExUB}\end{equation}

For the sake of contradiction, assume that $$\prob{S \sim \D^m \atop \phi \sim M(S)}{\exists j \in [k] ~~~~ S[\phi_j]-\D[\phi_j] > \alpha + \gamma \cdot \MAD\left(\phi_j(\D)\right)} > \beta.$$
It follows that $$\prob{S^1, \ldots, S^\ell \sim \D^m \atop \phi^1, \ldots,\phi^\ell \sim M(S^1), \ldots, M(S^\ell)}{\max_{(i,j) \in [\ell] \times [k]} S^i[\phi^i_j]-\D[\phi^i_j] - \gamma \cdot \MAD\left(\phi^i_j(\D)\right)> \alpha} > 1-(1-\beta)^\ell$$
and, hence, \begin{equation}\ex{S^1, \ldots, S^\ell \sim \D^m \atop \phi^1, \ldots, \phi^\ell \sim M(S^1), \ldots, M(S^\ell)}{\max_{(i,j) \in [\ell] \times [k] \cup \{(0,0)\}} S^i[\phi^i_j]-\D[\phi^i_j] - \gamma \cdot \MAD\left(\phi^i_j(\D)\right)} > \alpha(1-(1-\beta)^\ell).\label{eqn:ContraLB}
\end{equation}

Combining \eqref{eqn:EM}, \eqref{eqn:MultiExUB}, \eqref{eqn:ContraLB}, and $\gamma=e^{2\varepsilon}-1$ yields $$ \alpha(1-(1-\beta)^\ell) < \frac{2}{\varepsilon m} \ln(\ell \cdot k +1) + 2 \delta \ell.$$
Since $\ell \geq 1/\beta$, $1-(1-\beta)^\ell \geq 1-e^{-1} \geq 1/2$. We also have $\ell \leq 1/\beta + 1 \leq 2/\beta$. Thus $\frac{\alpha}{2} < \frac{2}{\varepsilon m} \ln\left(2\frac{k}{\beta}\right) + 4\frac{\delta}{\beta}$ --- a contradiction.
\end{proof}

\remove{

\section{Amplification by Median}
We state a useful and well-known lemma, which is relevant to our techniques:
\begin{lem}[Amplification by Median]\label{lem:amp-med}
Let $X_1, \ldots, X_\ell$ be i.i.d.~real random variables drawn from distribution $\D$. Fix $\rho \in (0,1/2)$. Then $$\prob{}{\median(X_1, \ldots, X_\ell) \in \iqr{\D}{\rho,1-\rho}} \geq 1 - 2 \cdot e^{- 2 (1/2-\rho)^2 \ell} .$$
\end{lem}
In particular, setting $\rho=1/4$, we have that $$\iqr{\median(\D^\ell)}{2 \cdot e^{-\ell/8},1-2 \cdot e^{-\ell/8}} \subseteq \iqr{\D}{\frac 1 4 , \frac 3 4 } .$$
The key point is that there is nothing special about $1/4$ and $3/4$ in Theorem \ref{thm:intro-iqr}. We can replace them with numbers very close to $0$ and $1$ at very little cost. 
\begin{proof}
First note that $\iqr{\D}{\rho,1-\rho} = \iqr{\D}{\rho,1} \cap \iqr{\D}{0,1-\rho}$. Thus, by a union bound, it suffices to prove that \begin{equation}\prob{}{\median(X_1, \ldots, X_\ell) \notin \iqr{\D}{\rho,1}} \leq e^{- 2 (1/2-\rho)^2 \ell} \label{eqn:rho1}\end{equation} and \begin{equation}\prob{}{\median(X_1, \ldots, X_\ell) \notin \iqr{\D}{0,1-\rho}} \leq e^{- 2 (1/2-\rho)^2 \ell} \label{eqn:rho0}.\end{equation}
Define $$A = \sum_{i\in[\ell]} \mathbb{I}[X_i \notin \iqr{\D}{\rho,1}]$$ so that $$\prob{}{\median(X_1, \ldots, X_\ell) \notin \iqr{\D}{\rho,1}} \leq \prob{}{A \geq \ell/2}. $$
By definition, $\ex{}{\mathbb{I}[X_i \notin \iqr{\D}{\rho,1}]}\leq\rho$ for all $i \in [\ell]$. Thus $A$ is the sum of $\ell$ independent random variables supported on $\{0,1\}$ each with mean $\leq \rho$, whence, by Hoeffding's inequality, $$\prob{}{A \geq \ell/2} \leq \prob{}{A-\ex{}{A} \geq \ell/2-\rho\ell} \leq e^{-2(\ell/2-\rho\ell)^2/\ell} = e^{-2(1/2-\rho)^2 \ell},$$ as required to prove \eqref{eqn:rho1}. The proof of \eqref{eqn:rho0} is symmetric.
\end{proof}

} 